\documentclass[11pt]{article}
\usepackage{latexsym,amssymb,amsmath} 
\usepackage{path}
\usepackage{url}
\usepackage{verbatim}

\usepackage{natbib}
\usepackage{amsfonts}
\usepackage{amsfonts}
\usepackage{array}
\usepackage{mathrsfs}
\usepackage{amsfonts}
\usepackage{amsmath}
\usepackage{amssymb}
\usepackage{pifont}
\usepackage{amsthm}
\usepackage{bm}
\usepackage{appendix}
\usepackage{multirow}
\usepackage{pifont}

\usepackage{color}
\usepackage[table]{xcolor}
\usepackage{thmtools}
\usepackage{thm-restate}
\usepackage{enumitem}
\usepackage{graphicx}
\usepackage{subfigure}

\usepackage[colorlinks=true,citecolor=blue]{hyperref}
\setcitestyle{open={(},close={)}}

\usepackage[lined,boxed,ruled, commentsnumbered]{algorithm2e}

\newtheorem{lemma}{Lemma}
\newtheorem{theorem}{Theorem}
\newtheorem{corollary}{Corollary}

\newtheorem{definition}{Definition}
\newtheorem{remark}{Remark}

\newcommand{\ms}{\texttt{L}_2}
\newcommand{\sgdw}{\texttt{SGD-w}}
\newcommand{\sgdwo}{\texttt{SGD-w/o}}
\newcommand{\opt}{\text{opt}}
\newcommand{\argmin}{\mathop{\arg\min}}


\numberwithin{equation}{section}

\title{Boosting the Confidence of Generalization for $L_2$-Stable Randomized Learning Algorithms}

\author{\vspace{0.4in}\\
  \textbf{Xiao-Tong Yuan} \ and \ \textbf{Ping Li} \\\\
  Cognitive Computing Lab \\
   Baidu Research \\
  No. 10 Xibeiwang East Road, Beijing 100193, China \\
  10900 NE 8th St. Bellevue, Washington 98004, USA\\
  E-mail: \texttt{\{xtyuan1980, pingli98\}@gmail.com}
  }
\date{}

\begin{document}

\maketitle

\begin{abstract}\vspace{0.3in}

\noindent Exponential generalization bounds with near-tight rates have recently been established for uniformly stable learning algorithms~\citep{feldman2019high,bousquet2020sharper}. The notion of uniform stability, however, is stringent in the sense that it is invariant to the data-generating distribution. Under the weaker and distribution dependent notions of stability such as hypothesis stability and $L_2$-stability, the literature suggests that only polynomial generalization bounds are possible in general cases. The present paper addresses this long standing tension between these two regimes of results and makes progress towards relaxing it inside a classic framework of confidence-boosting. To this end, we first establish an in-expectation first moment generalization error bound for potentially randomized learning algorithms with $L_2$-stability, based on which we then show that a properly designed subbagging process leads to near-tight exponential generalization bounds over the randomness of both data and algorithm. We further substantialize these generic results to stochastic gradient descent (SGD) to derive improved high-probability generalization bounds for convex or non-convex optimization problems with natural time decaying learning rates, which have not been possible to prove with the existing hypothesis stability or uniform stability based results.
\end{abstract}


\newpage

\section{Introduction}

In many statistical learning problems, the ultimate goal is to design a suitable algorithm $A: \mathcal{Z}^N \mapsto \mathcal{W}$ that maps a training data set $S=\{Z_i\}_{i\in [N]} \in \mathcal{Z}^N$ to a model $A(S)$ in a closed subset $\mathcal{W}$ of an Euclidean space such that the following population risk evaluated at the model is as small as possible:
\[
 R(A(S)):=\mathbb{E}_{Z \sim \mathcal{D}}[\ell (A(S);Z)].
\]
Here $\ell:\mathcal{W}\times \mathcal{Z} \mapsto \mathbb{R}^+$ is a non-negative bounded loss function whose value $\ell(w;z)$ measures the loss evaluated at $z$ with parameter $w$, and $\mathcal{D}$ represents a data distribution over $\mathcal{Z}$. It is generally the case that the underlying data distribution is unknown, and in this case the training data is usually assumed to be an i.i.d. set, i.e., $S \overset{\text{i.i.d.}}{\sim} \mathcal{D}^N$. Then, a natural alternative measurement that mimics the computationally intractable population risk is the empirical risk given by
\[
R_S(A(S)): =\mathbb{E}_{Z \sim \texttt{Unif}(S)}[\ell (A(S);Z)] = \frac{1}{N}\sum_{i=1}^N \ell(A(S); Z_i).
\]
The bound on the difference between population and empirical risks is of central interest in understanding the generalization performance of learning algorithm $A$. Particularly, we hope to derive a suitable law of large numbers, i.e., a sample size vanishing rate $b_N$ such that the generalization bound $|R_S(A(S)) - R(A(S))|\lesssim b_N$ holds with high probability over the randomness of $S$ and potentially the randomness of $A$ as well. Provided that $A(S)$ is an almost minimizer of the empirical risk function $R_S$, say $R_S(A(S)) \lesssim \min_{w\in \mathcal{W}} R_S(w) + \Delta_{\opt}$, the generalization bound immediately implies an \emph{excess risk} bound $R(A(S)) - \min_{w\in \mathcal{W}}R(w)\lesssim b_N +\Delta_{\opt}+ \frac{1}{\sqrt{N}}$ in view of the standard risk decomposition and Hoeffding's inequality. Therefore, generalization bounds also play a crucial role in understanding the stochastic optimization performance of a learning algorithm.

A powerful proxy for analyzing the generalization bounds is the \emph{stability} of learning algorithms to changes in the training dataset. Since the seminal work of~\citet{bousquet2002stability}, stability has been extensively demonstrated to beget dimension-independent generalization bounds for deterministic learning algorithms~\citep{mukherjee2006learning,shalev2010learnability}, as well as for randomized learning algorithms such as bagging and SGD~\citep{elisseeff2005stability,hardt2016train}. So far, the best known results about generalization bounds are offered by approaches based on the strongest notion of uniform stability~\citep{feldman2018generalization,feldman2019high,bousquet2020sharper} which is independent to the underlying distribution of data. Under the weaker and distribution dependent notions of stability such as hypothesis stability and $L_2$-stability, the literature suggests that only in-expectation or polynomial generalization bounds are possible in general cases~\citep{bousquet2002stability,hardt2016train}. Motivated by this gap between the two regimes of results, we seek to explore the opportunity of \emph{deriving exponential generalization bounds for randomized learning algorithms under $L_2$-stability}. A concrete working example of our study is the widely used stochastic gradient descent (SGD) algorithm that carries out the following recursion for all $t\ge 1$ with learning rate $\eta_t>0$:
\begin{equation}\label{equat:sgd}
 w_t:=\Pi_{\mathcal{W}} \left(w_{t-1} - \eta_t \nabla_w \ell (w_{t-1};Z_{\xi_t})\right),
\end{equation}
where $\xi_t \in [N]$ is a uniform random index of data under with or without replacement sampling, and $\Pi_{\mathcal{W}}$ is the Euclidean projection operator associated with $\mathcal{W}$. In spite of its popularity in the study of algorithmic stability theory~\citep{hardt2016train,lei2020fine,zhou2019understanding}, the exponential generalization bounds of SGD are still under explored through the lens of algorithmic~stability.

\subsection{Prior Results}

Let us start by briefly reviewing some state-of-the-art exponential generalization bounds under the most stringent notion of uniform stability and its randomized variants. We denote by $S \doteq \tilde S$ if a pair of data sets $S$ and $\tilde S$ differ in a single data point. A randomized learning algorithm $A$ is said to have on-average $\gamma_{u,N}$-uniform stability if it satisfies the following uniform bound
\[
\sup_{S \doteq \tilde S,Z\in \mathcal{Z}}|\mathbb{E}_A[\ell(A(S);Z)] - \mathbb{E}_A[\ell(A(\tilde S);Z)]| \le \gamma_{u,N}.
\]
This definition is equivalent to the concept of uniform stability defined over the on-average loss $\mathbb{E}_A[\ell(A(S);Z)]$. Suppose that the loss function is uniformly bounded in the interval $(0,1]$. Then essentially it has been shown in~\citet{feldman2019high} that for any $\delta\in(0,1)$, with probability at least $1-\delta$ over $S$, the on-average generalization error is upper bounded by
\begin{equation}\label{inequat:feldman_2019}
\left|\mathbb{E}_A\left[R(A(S)) - R_S(A(S))\right]\right| \lesssim \gamma_{u,N} \log(N) \log\left(\frac{N}{\delta}\right) + \sqrt{\frac{\log\left(1/\delta\right)}{N}}.
\end{equation}
Recently, \citet{bousquet2020sharper} derived a slightly improved uniform stability bound that implies
\begin{equation}\label{inequat:bousquet_2020}
\left|\mathbb{E}_A\left[R(A(S)) - R_S(A(S))\right]\right|\lesssim \gamma_{u,N} \log(N) \log\left(\frac{1}{\delta}\right) + \sqrt{\frac{\log\left(1/\delta\right)}{N}}.
\end{equation}
These generalization bounds are near-tight (up to a logarithmic factor $\log(N)$) in the sense of an $\mathcal{O}\big(\gamma_{u,N} \log\left(\frac{1}{\delta}\right) + \sqrt{\frac{\log\left(1/\delta\right)}{N}}\big)$ lower deviation bound on sum of random functions with $\gamma_{u,N}$-uniform stability~\citep[Proposition 9]{bousquet2020sharper}. While sharp in the dependence on sample size, one common limitation of the above uniform stability implied generalization bounds lies in that these high-probability results only hold \emph{in expectation} with respect to the internal randomness of algorithm.

Further suppose that $A$ has $\gamma_{u,N}$-uniform stability with probability at least $1-\delta'$ for some $\delta'\in(0,1)$ over the randomness of $A$, i.e.,
\begin{equation}\label{inequat:uniform_stability_highprob}
\mathbb{P}_A\left\{\sup_{S \doteq \tilde S,z\in \mathcal{Z}}|\ell(A(S);z) - \ell(A(\tilde S);z)| \le \gamma_{u,N} \right\} \ge 1-\delta'.
\end{equation}
Suppose that the randomness of $A$ is independent of the training set $S$. Then with probability at least $1-\delta-\delta'$ over $S$ and $A$, the bound of~\citet{bousquet2020sharper} naturally implies
\begin{equation}\label{inequat:bousquet_2020_algrand}
\left|R(A(S)) - R_S(A(S))\right| \lesssim \gamma_{u,N} \log(N) \log\left(\frac{1}{\delta}\right) + \sqrt{\frac{\log\left(1/\delta\right)}{N}}.
\end{equation}
This is by far the best known generalization bound of randomized stable algorithms that hold with high probability jointly over data and algorithm. The result, however, relies heavily on the high-probability uniform stability condition expressed in~\eqref{inequat:uniform_stability_highprob}. For the SGD (see~\eqref{equat:sgd}) with fixed learning rate $\eta_t\equiv \eta$, it is possible to show that $\gamma_{u,N} \lesssim \eta\sqrt{T} + \frac{\eta T}{N}$ and $\delta'=N\exp(-\frac{N}{2})$ in~\eqref{inequat:uniform_stability_highprob}~\citep{bassily2020stability}. For SGD with time decaying learning rate that has been widely studied in theory~\citep{rakhlin2012making,harvey2019tight} and applied in practice for training popular deep nets such as ResNet and DenseNet~\citep{bengio2017deep}, it is not clear if the condition in~\eqref{inequat:uniform_stability_highprob} is still valid for $\gamma_{u,N}$ and $\delta'$ of interest.
In the meanwhile, it is possible to show (see the proofs of Corollary~\ref{corol:stability_sgd_w_convex_smooth} and~\ref{corol:stability_sgd_w_nsm_cvx}) that SGD with time decaying learning rate has $L_2$-stability (see Definition~\ref{def:mean_uniform_stability}) which is distribution dependent and much weaker than uniform stability. However, through the notions of hypothesis stability or $L_2$-stability, only in-expectation or $\mathcal{O}(1/\delta)$-polynomial-tail generalization bounds are available in literature~\citep{bousquet2002stability,hardt2016train}.

More specially for randomized learning methods such as bagging~\citep{breiman1996bagging} and SGD~\eqref{equat:sgd}, the randomness of algorithm can be precisely characterized by a vector of i.i.d. parameters $\xi=\{\xi_1,...,\xi_T\}$ which are independent on data $S$. In such cases, suppose that $A(S;\xi)$ has uniform stability with respect to $\xi$ at any given $S$, i.e., $\sup_{\xi \doteq \xi'} |\ell(A(S;\xi)) - \ell(A(S;\xi'))| \le \rho_{u,T}$. Then the high probability bound established in~\citet{elisseeff2005stability} shows that with probability at least $1-\delta$,
\begin{equation}\label{inequat:elisseeff_2005}
\left|R(A(S)) - R_S(A(S))\right| \lesssim \gamma_{u,N} + \left(\frac{1 + N\gamma_{u,N}}{\sqrt{N}} + \sqrt{T}\rho_{u,T}\right)\sqrt{\log\left(\frac{1}{\delta}\right)}.
\end{equation}
Provided that $\gamma_{u,N} \lesssim \frac{1}{N}$ and $\rho_{u,T} \lesssim \frac{1}{T}$, the above bound shows that the generalization bound scales as $\mathcal{O}\big(\frac{1}{\sqrt{N}} + \frac{1}{\sqrt{T}}\big)$ with high probability. However, the above bound will show no guarantee on convergence if $\gamma_{u,N} \gtrsim \frac{1}{\sqrt{N}}$ and/or $\rho_{u,T}\gtrsim \frac{1}{\sqrt{T}}$. For example, this is actually the case for SGD with decaying learning rate $\eta_t=O\big(\frac{1}{t}\big)$ on non-convex loss functions in which $\gamma_{u,N}\lesssim \frac{\sqrt{T}}{N}$ and $\rho_{u,T}$ can scale as large as $\mathcal{O}(1)$.

\vspace{0.2in}

\noindent\textbf{Open problem.} Keeping the merits and deficiencies of above recalled prior results in mind, it still remains an open issue if the exponential generalization bounds of uniformly stable algorithms might possibly be obtained for randomized algorithms under weaker and distribution dependent notions of stability like $L_2$-stability (e.g., fulfilled by SGD with time decaying learning rates). The focus of the present study is to explore the opportunity of answering this open question affirmatively.

\subsection{Overview of Our Results}
\label{ssect:overview}

The confidence-boosting technique of~\citet{schapire1990strength} is a classic meta approach that allows us to boost the dependence of a learning algorithm on the failure probability $\delta$ from $1/\delta$ to $\log(1/\delta)$, at a certain cost of computational complexity. The fundamental contribution of our work is to reveal that the confidence-boosting trick yields near-tight exponential generalization bounds for $L_2$-stable randomized learning algorithms. The novelty lies in a refined analysis of the in-expectation first moment generalization error bound for a randomized learning algorithm with $L_2$-stability, which leads to desirable exponential generalization bounds over the randomness of data and algorithm via confidence-boosting. More specifically, given a randomized learning algorithm $A$, we propose to study a subbagging based confidence-boosting procedure as following described.

\textbf{Boosting the confidence via subbagging}. We independently run $A$ over $K$ \emph{disjoint} and uniformly divided training subsets $\{S_k\}_{k\in [K]}$ to obtain solutions $\{A_k(S_k)\}_{k\in [K]}$. Then we evaluate the validation error of each candidate solution over its complementary training subset, and output $A_{k^*}(S_{k^*})$ that has the smallest gap between training error and validation error, i.e., $k^* = \argmin_{k\in [K]} \left|R_{S\setminus S_k }(A_k(S_k)) - R_{S_k}(A_k(S_k)) \right|$. Specially when $A$ is deterministic, this reduces to a standard subbagging process, namely a variation of bagging using without-replacement sampling for subsets generation~\citep[see, e.g.,][]{andonova2002simple}. In general, this is essentially a subbagging procedure with greedy model selection for randomized algorithms over multiple disjoint subsets. Throughout this paper, we assume without loss of generality that $N$ is a multiplier of $K$. The considered procedure of confidence-boosting for randomized learning is outlined in Algorithm~\ref{alg:randomized_model_selection}.

\begin{algorithm}[t]
\caption{Confidence-Boosting for Randomized Learning Algorithms}
\label{alg:randomized_model_selection}
\SetKwInOut{Input}{Input}\SetKwInOut{Output}{Output}\SetKw{Initialization}{Initialization}
\Input{A randomized learning algorithm $A$ and a training data set $S=\{Z_i\}_{i\in [N]} \overset{\text{i.i.d.}}{\sim} \mathcal{D}^N$.}
\Output{$A_{k^*}(S_{k^*})$.}

Uniformly divide $S$ into $K$ \emph{disjoint} subsets such that $S=\bigcup _{k\in [K]} S_k$ and $|S_k|= \frac{N}{K}$, $\forall k\in [K]$.

\For{$k=1, 2, ...,K$}{
Estimate $A_k(S_k)$ as an output of the randomized algorithm $A$ over subset $S_k$.
}

Select the subset index $k^*$ according to
$
k^* = \argmin_{k\in [K]} \left|R_{S\setminus S_k }(A_k(S_k)) - R_{S_k}(A_k(S_k)) \right|.
$
\end{algorithm}

\textbf{Main results.} In what follows, we highlight our main results on the generalization bounds of the output of Algorithm~\ref{alg:randomized_model_selection} along with the implications for SGD:
\begin{itemize}[leftmargin=*]
  \item \emph{General results.}  Suppose that the loss is Lipschitz and upper bounded. Our main result in Theorem~\ref{thrm:stability_genalization_error_highprob_bagging} show that for any $\delta\in(0,1)$, setting $K\asymp\log(\frac{1}{\delta})$ yields the following generalization bound of the output of Algorithm~\ref{alg:randomized_model_selection} that holds with probability at least $1-\delta$ over $S$ and $\{A_k\}_{k\in[K]}$:
\[
\left|R(A_{k^*}(S_{k^*})) - R_S(A_{k^*}(S_{k^*}))\right| \lesssim \frac{1}{K}\left(\gamma_{\ms,\frac{N}{K}} + \sqrt{\frac{K}{N}} \right) + \sqrt{\frac{\log(K/\delta)}{N}},
\]
where $\gamma_{\ms, N}$ is the $L_2$-stability parameter introduced in Definition~\ref{def:mean_uniform_stability}. In contrast to the bounds in~\eqref{inequat:bousquet_2020_algrand} and~\eqref{inequat:elisseeff_2005}, our bound is not relying on the uniform stability of algorithm with respect to data or internal random bits.
  \item \emph{Sharper deviation bounds for SGD via confidence-boosting.} We then use our general results to study the benefit of confidence-boosting on the generalization bounds of \texttt{SGD-w} (SGD via with-replacement sampling as outlined in Algorithm~\ref{alg:sgd_w}). The main results are a series of corollaries of Theorem~\ref{thrm:stability_genalization_error_highprob_bagging} when substantialized to SGD with smooth~(Corollary~\ref{corol:stability_sgd_w_convex_smooth}) or non-smooth (Corollary~\ref{corol:stability_sgd_w_nsm_cvx}) convex losses, and smooth non-convex losses (Corollary~\ref{corol:stability_sgd_w_sm_ncvx}) as well. For an instance, our result in Corollary~\ref{corol:stability_sgd_w_convex_smooth} showcases that when invoked to \texttt{SGD-w} on smooth convex loss with learning rates $\eta_t=\mathcal{O}(\frac{1}{\sqrt{t}})$, the generalization bound of the output of Algorithm~\ref{alg:randomized_model_selection} with $K\asymp\log(\frac{1}{\delta})$ is upper bounded by
\[
\left|R(A_{\sgdw,k^*}(S_{k^*})) - R_S(A_{\sgdw,k^*}(S_{k^*}))\right| \lesssim \sqrt{\frac{\log(T)}{N}} + \frac{\sqrt{T}+\sqrt{N}\log(1/\delta)}{N}.
\]
Compared with the $\mathcal{O}\big(\frac{\sqrt{T}}{N}\big)$ in-expectation bound of smooth convex SGD~\citep{hardt2016train}, our above bound is comparable for $T=\mathcal{O}(N)$ while it holds with high confidence over the randomness of data and stochastic sampling.
\end{itemize}
In addition to the generalization bounds, we have also derived an exponential excess risk bound for $L_2$-stable randomized learning with confidence-boosting. More specifically, with a proper modification of the output of Algorithm~\ref{alg:randomized_model_selection}, we can show that with probability at least $1-\delta$ over $S$ and $\{A_k\}_{k\in[K]}$:
\[
 R(A_{k^*}(S_{k^*})) - \min_{w\in \mathcal{W}} R(w) \lesssim \gamma_{\ms,\frac{N}{K}} + \Delta_{\opt} + \sqrt{\frac{\log(K/\delta)}{N}},
\]
where $\Delta_{\opt}$ is the in-expectation empirical risk minimization sub-optimality as given by~\eqref{equat:Delta_Opt}.

\section{Exponential Generalization Bounds for Confidence-Boosting}
\label{sect:general_results}

In this section, we present a set of generic results on the generalization bounds of randomized learning algorithms with confidence-boosting as described in Algorithm~\ref{alg:randomized_model_selection}.

\subsection{Preliminaries and a Key Lemma}

\textbf{Notation.} Recall that $S=\{Z_1, Z_2, ...,Z_N\}$ is a set of i.i.d. training data points. Denote by $S'=\{Z'_1,Z'_2,..., Z'_N\}$ another i.i.d. sample from the same distribution as that of $S$ and we write $S^{(i)} = \{Z_1,...,Z_{i-1},Z'_i,Z_{i+1},...,Z_N\}$. For a real-valued random variable $Y$, its $L_q$-norm for $q\ge 1$ is given by $ \|Y\|_q = \left(\mathbb{E}[|Y|^q]\right)^{1/q}$. We say a function $f$ is $G$-Lipschitz continuous over $\mathcal{W}$ if $|f(w) - f(w')|\le G\|w - w'\|$ for all $w, w'\in \mathcal{W}$, and it is $L$-smooth if $\|\nabla f(w) - \nabla f(w')\|\le L\|w - w'\|$ for all $w, w'\in \mathcal{W}$. We abbreviate $[N]:=\{1,...,N\}$.

Let us first introduce the following notion of $L_2$-stability due to~\citet{kale2011cross}, which extends the weak and distribution dependent hypothesis stability~\citep{bousquet2002stability} from $L_1$-norm criterion to $L_2$-norm criterion.
\begin{definition}[$L_2$-Stability of Randomized Learning Algorithms]\label{def:mean_uniform_stability}
Let $A: \mathcal{Z}^N \mapsto \mathcal{W}$ be a randomized learning algorithm that maps a data set $S \in \mathcal{Z}^N$ to a model $A(S) \in \mathcal{W}$. Then $A$ is said to have $L_2$-stability by $\gamma_{\ms,N}$ if for all $N\ge 1$ and $i\in [N]$,
\[
\left\|\ell(A(S);Z) -\ell( A(S^{(i)});Z)\right\|_2 \le \gamma_{\ms,N},
\]
where the expectation associated with $\|\cdot\|_2$ is taken over the randomness of $A, S,S^{(i)}$ and $Z\in \mathcal{Z}$, and the algorithm outputs $A(S)$ and $A(S^{(i)})$ share the same random bits associated with the algorithm.
\end{definition}

Inspired by a second moment bound for generalization error of uniformly stable algorithms from~\citet[Section 5]{bousquet2020sharper}, we first establish the following lemma which states that if a randomized learning algorithm has $L_2$-stability by $\gamma_{\ms,N}$, then its on-average first moment generalization error bound will be as small as $\mathcal{O}\big(\gamma_{\ms,N} + \frac{1}{\sqrt{N}}\big)$. For the sake of completeness, we provide a proof for this result in Appendix~\ref{apdsect:proof_lemma_stability_generalization_error_average}.
\begin{lemma}\label{lemma:stability_genalization_first_order_moment}
Suppose that a randomized learning algorithm $A: \mathcal{Z}^N \mapsto \mathcal{W}$ has $L_2$-stability by $\gamma_{\ms,N}$. Assume that the loss function satisfies $\|\ell(A(S);Z)\|_2 \le M$. Then we have
\[
\mathbb{E}_{A,S} \left[|R(A(S)) - R_S(A(S))|\right] \le  3\gamma_{\ms,N} + \frac{2M}{\sqrt{N}}.
\]
\end{lemma}

\begin{remark}
Under the notion of on-average uniform stability, it is well-known that the in-expectation generalization error can be bounded as $\left|\mathbb{E}_{A,S} \left[R(A(S)) - R_S(A(S))\right]\right| \lesssim \gamma_{u,N}$~\citep[see, e.g.,][Theorem 2.2]{hardt2016train}. In comparison, our on-average first moment bound in Lemma~\ref{lemma:stability_genalization_first_order_moment} is substantially stronger, despite that the $L_2$-stability condition is weaker. This result turns out to play a crucial role in deriving high probability bounds for randomized algorithms. When $A$ is deterministic, the above bound reduces to the explicitly or implicitly known first moment bound for uniformly stable algorithms~\citep[see, e.g.,][]{feldman2018generalization,bousquet2020sharper}. In this deterministic case, our bound is stronger than the bound from~\citet[Lemma 9]{bousquet2002stability} which essentially scales as $\mathcal{O}\left(\frac{1}{\sqrt{N}} + \sqrt{\gamma_{\ms,N}}\right)$ using our notation.
\end{remark}

\subsection{Main Results}

Let us recall the subbagging process as described in Algorithm~\ref{alg:randomized_model_selection}: we independently run $A$ over $K$ even and disjoint training subsets $\{S_k\}_{k\in [K]}$ to obtain solutions $\{A_k(S_k)\}_{k\in [K]}$, and then pick $A_{k^*}(S_{k^*})$ that has the smallest difference between training error and validation error (over the complementary training subset $S\setminus S_{k^*}$). The following theorem is our main result about the high probability generalization bound of the output $A_{k^*}(S_{k^*})$ evaluated over the entire training set $S$. See Appendix~\ref{apdsect:proof_thrm_stability_genalization_error_highprob_bagging} for its proof which builds largely on the first moment bound in Lemma~\ref{lemma:stability_genalization_first_order_moment} and the fact that at least one of the solutions generated by subbagging generalizes well with high probability.
\begin{theorem}\label{thrm:stability_genalization_error_highprob_bagging}
Suppose that a randomized learning algorithm $A: \mathcal{Z}^N \mapsto \mathcal{W}$ has $L_2$-stability by $\gamma_{\ms,N}$. Assume that the loss function $\ell$ is bounded in the range of $[0,M]$. Then for any $\alpha,\delta \in (0,1)$ and $K \ge\frac{1}{1-\alpha} \log(\frac{4}{\delta})$,  with probability at least $1 - \delta$ over the randomness of $S$ and $\{A_k\}_{k\in [K]}$, the output of Algorithm~\ref{alg:randomized_model_selection} satisfies
\[
\left|R(A_{k^*}(S_{k^*})) - R_S(A_{k^*}(S_{k^*}))\right| \lesssim \frac{1}{\alpha K}\left(\gamma_{\ms,\frac{N}{K}} + M\sqrt{\frac{K}{N}}\right) + M\sqrt{\frac{\log(K/\delta)}{N}}.
\]
\end{theorem}
\begin{proof}[Proof in sketch]
In light of the first moment bound in Lemma~\ref{lemma:stability_genalization_first_order_moment}, we first establish a key intermediate result, Lemma~\ref{thrm:stability_genalization_error_highprob}, which confirms that for large enough $K$, at least one of the $K$ base models generated by subbagging generalizes well with high probability over data. To be more precise, under the given conditions, Lemma~\ref{thrm:stability_genalization_error_highprob} shows that if $K\gtrsim \frac{1}{1-\alpha}\log(\frac{1}{\delta})$, then
$\min_{k\in [K]} \left|R(A_k(S_k)) - R_{S_k}(A_k(S_k))\right|$ can be well upper bounded with high probability in terms of $\gamma_{\ms, \frac{N}{K}}$. Then we show that the proposed greedy model selection strategy guarantees that the selected $A_{k^*}(S_{k^*})$ mimics the generalization behavior of that best performer among the $K$ candidates, also with high probability. Finally, the desired bound follows immediately via union probability. See Appendix~\ref{apdsect:proof_thrm_stability_genalization_error_highprob_bagging} for a full proof.
\end{proof}

\begin{remark}
To gain some intuition on the superiority of our bound in Theorem~\ref{thrm:stability_genalization_error_highprob_bagging}, let us consider $K \asymp \log\left(\frac{1}{\delta}\right)$ as allowed in the conditions. If $\gamma_{\ms,N}\lesssim \frac{1}{\sqrt{N}}$, then our high-probability bound in Theorem~\ref{thrm:stability_genalization_error_highprob_bagging} roughly scales as $\mathcal{O}\left(\sqrt{\frac{\log(1/\delta)}{N}}\right)$ which is sharper than  the on-average bounds in~\eqref{inequat:feldman_2019} and~\eqref{inequat:bousquet_2020} with $\gamma_{u,N}\lesssim\frac{1}{\sqrt{N}}$. More precise consequences of these general results on SGD will be discussed shortly in the sections to follow.
\end{remark}

\newpage

\begin{remark}
In sharp contrast to the bound in~\eqref{inequat:bousquet_2020_algrand} that requires high probability uniform stability and the bound in~\eqref{inequat:elisseeff_2005} that assumes uniform stability over the random bits of algorithm, our bound in Theorem~\ref{thrm:stability_genalization_error_highprob_bagging} holds under a substantially milder notion of $L_2$-stability over data. In terms of the tightness of bound, note that the confidence term $\sqrt{\frac{\log(1/\delta)}{N}}$ is necessary even for an algorithm with fixed output. The stability term $\gamma_{\ms,\frac{N}{K}}$ are also near-tight as the algorithm output can change arbitrarily with respect to the quantity.
\end{remark}

Let us consider a specification of Theorem~\ref{thrm:stability_genalization_error_highprob_bagging} to a deterministic algorithm $A$ with $L_2$-stability. Since there is no randomness contained in $A$, we have that Algorithm~\ref{alg:randomized_model_selection} outputs $A_k(S_k)=A(S_k)$ for all $k\in [K]$ in such a deterministic case. In the regime $\gamma_{\ms,N} \lesssim \frac{1}{\sqrt{N}}$ which is of interest in many popular deterministic learning algorithms such as regularized ERM~\citep{shalev2009stochastic} and full gradient descent~\citep{feldman2019high}, the generalization bound in Theorem~\ref{thrm:stability_genalization_error_highprob_bagging} with $\alpha = 0.5$ and $K\asymp \log(\frac{1}{\delta})$ implies a generalization bound for $A(S_{k^*})$ over the data set $S$ that scales as $ \left|R(A(S_{k^*})) - R_S(A(S_{k^*}))\right| \lesssim \sqrt{\frac{\log\left(1/\delta\right)}{N}}$. In comparison, the best known bound in~\eqref{inequat:bousquet_2020_algrand} gives $\left|R(A(S)) - R_S(A(S))\right| \lesssim \frac{\log(N)}{\sqrt{N}} \log(\frac{1}{\delta})$ (keep in mind that $\delta'=0$ in the deterministic case)~\citet{bousquet2020sharper}. As we can see that inside the carefully designed framework of confidence-boosting via subbagging, our generalization bound gets rid of the logarithmic factor $\log(N)$ from the above best known result, though the generalization is with respect to the estimation over a specific part of the training sample. We expect this result will fuel future research towards fully resolving the corresponding open question raised by~\citet{bousquet2020sharper}.

The result in Theorem~\ref{thrm:stability_genalization_error_highprob_bagging} is about the generalization bound with respect to the empirical risk over the entire training set $S$. We comment in passing that up to the multipliers of $K$, these bounds are also valid for the generalization error $\left|R(A_{k^*}(S_{k^*})) - R_{S_{k^*}}(A_{k^*}(S_{k^*}))\right|$ over the subset $S_{k^*}$.

\subsection{On Excess Risk Bounds}

To understand the optimization performance of a randomized learning algorithm $A$ with confidence-boosting, we further study here the excess risk bounds of Algorithm~\ref{alg:randomized_model_selection} which are of special interest for stochastic convex optimization problems. In the following analysis, the global minimizer of the population risk and in-expectation empirical risk sub-optimality of the randomized algorithm are respectively denoted by $w^*:=\argmin_{w \in \mathcal{W}}R(w) $ and
\begin{equation}\label{equat:Delta_Opt}
\Delta_{\opt} := \mathbb{E}_{A,S} \left[R_S(A(S)) - \min_{w\in \mathcal{W}} R_S(w)\right].
\end{equation}
Throughout this paper , we assume that the sub-optimality $\Delta_{\opt}$ is invariant to the size of training data $S$. In order to derive the excess risk guarantees, we first need to slightly modify the output of Algorithm~\ref{alg:randomized_model_selection} as $A_{k^*}(S_{k^*})$ with
\begin{equation}\label{equat:output_modify_subbagging}
 k^* = \argmin_{k\in [K]} R_{S\setminus S_k }(A_k(S_k)).
\end{equation}
 The following theorem is our main result about the high probability excess risk bounds of such a modified output of confidence-boosting. See Appendix~\ref{apdsect:proof_thrm_stability_excess_risk_highprob_bagging} for its proof.
 
\newpage
 
\begin{theorem}\label{thrm:stability_excess_risk_highprob_bagging}
Suppose that a randomized learning algorithm $A: \mathcal{Z}^N \mapsto \mathcal{W}$ has $L_2$-stability by $\gamma_{\ms,N}$. Assume that the loss function $\ell$ is bounded in $[0,M]$. Then for any $\alpha,\delta \in (0,1)$ and $K \ge\frac{1}{1-\alpha} \log(\frac{4}{\delta})$,  with probability at least $1 - \delta$ over the randomness of $S$ and $\{A_k\}_{k\in [K]}$, the modified output of Algorithm~\ref{alg:randomized_model_selection} given by~\eqref{equat:output_modify_subbagging} satisfies
\[
R(A_{k^*}(S_{k^*})) - R(w^*) \lesssim \frac{1}{\alpha}\left( \gamma_{\ms,\frac{N}{K}} + \Delta_{\opt} \right)  + M\sqrt{\frac{\log(K/\delta)}{N}} .
\]
\end{theorem}
\begin{remark}
Consider $K \asymp \log\left(\frac{1}{\delta}\right)$ and $\gamma_{\ms,N}\lesssim \frac{1}{\sqrt{N}}$. Then the bound in Theorem~\ref{thrm:stability_excess_risk_highprob_bagging} is roughly $\mathcal{O}\big(\sqrt{\frac{\log(1/\delta)}{N}}+\Delta_{\opt}\big)$. Compared with the generalization error bound in Theorem~\ref{thrm:stability_genalization_error_highprob_bagging}, the excess risk bound established in Theorem~\ref{thrm:stability_excess_risk_highprob_bagging} has an additional term of in-expectation optimization error for minimizing the empirical risk. For deterministic optimization algorithms such as ERMs with $\Delta_{\opt}=0$, similar excess risk bounds can be implied by the generic results of~\citet[Theorem 26]{shalev2010learnability} developed for the confidence-boosting approach.
\end{remark}
Finally, we comment on the difference between the generalization error and excess risk analysis inside the considered confidence-boosting framework. Since the excess risk is always non-negative and its in-expectation bound more or less standardly known under $L_2$-stability , the high-confidence bound in Theorem~\ref{thrm:stability_excess_risk_highprob_bagging} can be easily derived via invoking Markov inequality to the independent runs of algorithm over $K$ disjoint subsets. The generalization error analysis, however, is significantly more challenging in the sense that establishing tight in-expectation first moment generalization bound (see Lemma~\ref{lemma:stability_genalization_first_order_moment}) for $L_2$-stable randomized algorithms is by itself non-trivial.

\section{Implications for SGD}
\label{sect:implication_sgd}

In this section we demonstrate the applications of the generic bounds in Theorem~\ref{thrm:stability_genalization_error_highprob_bagging} and Theorem~\ref{thrm:stability_excess_risk_highprob_bagging} to the widely used SGD algorithm. We focus on a variant of SGD under with-replacement sampling as outlined in Algorithm~\ref{alg:sgd_w}, which we call $A_{\sgdw}$. In what follows, we denote by $\{A_{\sgdw,k}\}_{k\in K}$ the outputs of $A_{\sgdw}$ over subsets $\{S_k\}_{k\in K}$ when implemented with Algorithm~\ref{alg:randomized_model_selection}. Our results readily extend to the without-replacement variant of SGD and the corresponding results can be found in Appendix~\ref{apdsect:results_for_SGD_wo}.

\begin{algorithm}[t]\caption{$A_{\sgdw}$: SGD under With-Replacement Sampling}
\label{alg:sgd_w}
\SetKwInOut{Input}{Input}\SetKwInOut{Output}{Output}\SetKw{Initialization}{Initialization}
\Input{Data set $S=\{Z_i\}_{i\in [N]} \overset{\text{i.i.d.}}{\sim} \mathcal{D}^N$, step-sizes $\{\eta_t\}_{t\ge 1}$, \#iterations $T$, initialization $w_0$.}
\Output{$\bar w_{T}=\frac{1}{T}\sum_{t\in [T]}w_t$.}

\For{$t=1, 2, ...,T$}{

Uniformly randomly sample an index $\xi_t\in [N]$ with replacement;

Compute $w_t = \Pi_{\mathcal{W}}\left( w_{t-1} - \eta_t \nabla_w \ell(w_{t-1}; Z_{\xi_t})\right)$.
}
\end{algorithm}

\newpage

\subsection{Convex Optimization with Smooth Loss}

For smooth and convex losses such as logistic loss, we can derive the following result as a direct consequence of Theorem~\ref{thrm:stability_genalization_error_highprob_bagging} with $\alpha=1/2$ when applied to $A_{\sgdw}$. See Appendix~\ref{apdsect:proof_lemma:uniform_stability_sgd_w_sm_cvx} for its proof.
\begin{corollary}\label{corol:stability_sgd_w_convex_smooth}
Suppose that the loss function is $\ell(\cdot; \cdot)$ is convex, $G$-Lipschitz and $L$-smooth with respect to its first argument, and is bounded in $[0,M]$. Consider Algorithm~\ref{alg:randomized_model_selection} specified to $A_{\sgdw}$ with learning rate $\eta_t\le 2/L$ for $t\ge 1$. Then for any $\delta \in (0,1)$ and $K \asymp\log(\frac{4}{\delta}) $, with probability at least $1 - \delta$ over the randomness of $S$ and $\{A_{\sgdw,k}\}_{k\in [K]}$, the generalization error is upper bounded as
\[
\left|R(A_{\sgdw,k^*}(S_{k^*})) - R_S(A_{\sgdw,k^*}(S_{k^*}))\right| \lesssim G^2\sqrt{\frac{1}{N}\left(\sum_{t=1}^T \eta^2_t + \frac{1}{N} \left(\sum_{t=1}^T \eta_t\right)^2\right)} + M\sqrt{\frac{\log(1/\delta)}{N}}.
\]
\end{corollary}
\begin{remark}
For the conventional step-size choice of $\eta_t = \frac{2}{L\sqrt{t}}$, the high probability generalization bound in Corollary~\ref{corol:stability_sgd_w_convex_smooth} is of scale $
\mathcal{O}\big(\sqrt{\frac{\log(T)}{N}} + \frac{\sqrt{T}+\sqrt{N\log(1/\delta)}}{N}\big)$, which matches the corresponding $\mathcal{O}(\frac{\sqrt{T}}{N})$ in-expectation bound for SGD with smooth and convex losses~\citep{hardt2016train}.
\end{remark}
Combining with the standard in-expectation optimization error bound of convex SGD~\citep[see, e.g.,][]{shamir2013stochastic}, we can show the following excess risk bound as a direct consequence of Theorem~\ref{thrm:stability_excess_risk_highprob_bagging} to $A_{\sgdw}$ with convex and smooth losses:
\[
\begin{aligned}
&R(A_{\sgdw, k^*}(S_{k^*})) - R(w^*) \\
\lesssim& G^2\sqrt{\frac{1}{N}\left(\sum_{t=1}^T \eta^2_t + \frac{1}{N} \left(\sum_{t=1}^T \eta_t\right)^2\right)} + M\sqrt{\frac{\log(1/\delta)}{N}} + \frac{\|w_0 - w^*\|^2+G^2\sum_{t=1}^T \eta_t^2}{\sum_{t=1}^T\eta_t}.
\end{aligned}
\]
With learning rate $\eta_t = \frac{2}{L\sqrt{t}}$, the right hand side of the above scales as $\mathcal{O}\big(\sqrt{\frac{\log(T)}{N}} + \frac{\sqrt{T}+\sqrt{N\log(1/\delta)}}{N} + \frac{\log(T)}{\sqrt{T}}\big)$, which matches the high-probability excess risk bonds of convex SGD by~\citet[Remark 3.7]{harvey2019tight}.

\subsection{Convex Optimization with Non-smooth Loss}

Now we turn to study the case where the loss is convex but not necessarily smooth, such as the hinge loss and absolute loss. The following result as a direct consequence of Theorem~\ref{thrm:stability_genalization_error_highprob_bagging} for the specification of Algorithm~\ref{alg:randomized_model_selection} to $A_{\sgdw}$ with non-smooth convex loss and time varying learning rate $\{\eta_t\}_{t\ge 1}$. Its proof is provided in Appendix~\ref{apdsect:proof_lemma:uniform_stability_sgd_w_nsm_cvx}.
\begin{corollary}\label{corol:stability_sgd_w_nsm_cvx}
Suppose that the loss function is $\ell(\cdot; \cdot)$ is convex and $G$-Lipschitz with respect to its first argument, and is bounded in the range of $[0,M]$. Consider Algorithm~\ref{alg:randomized_model_selection} specified to $A_{\sgdw}$. Then for any $\delta \in (0,1)$ and $K \asymp \log(\frac{4}{\delta}) $,  with probability at least $1 - \delta$ over the randomness of $S$ and $\{A_{\sgdw,k}\}_{k\in [K]}$,  the generalization error satisfies
\[
\left|R(A_{\sgdw,k^*}(S_{k^*})) - R_S(A_{\sgdw,k^*}(S_{k^*}))\right| \lesssim G^2\sqrt{\sum_{t=1}^{T}\eta^2_t + \frac{1}{N^2} \left(\sum_{t=1}^T \eta_t\right)^2} + M\sqrt{\frac{\log(1/\delta)}{N}}.
\]
\end{corollary}
\begin{remark}
For constant learning rates $\eta_t\equiv \eta$, Corollary~\ref{corol:stability_sgd_w_nsm_cvx} admits a high-probability generalization bound of scale $\mathcal{O}\big(\eta\sqrt{T}+\eta\frac{T}{N} + \sqrt{\frac{\log(1/\delta)}{N}}\big)$ which matches the near-optimal rate by~\citet[Theorem 3.3]{bassily2020stability}.
\end{remark}

\subsection{Non-convex Optimization with Smooth Loss}

We further study the performance of Algorithm~\ref{alg:randomized_model_selection} for SGD on smooth but not necessarily convex loss functions, such as normalized sigmoid loss~\citep{mason1999boosting}. The following result is a direct application of Theorem~\ref{thrm:stability_genalization_error_highprob_bagging} to $A_{\sgdw}$ with smooth non-convex loss. See Appendix~\ref{apdsect:proof_lemma:uniform_stability_sgd_w_sm_ncvx} for its proof.
\begin{corollary}\label{corol:stability_sgd_w_sm_ncvx}
Suppose that the loss function is $\ell(\cdot; \cdot)$ is $G$-Lipschitz and $L$-smooth with respect to its first argument, and is bounded in $[0,M]$. Consider Algorithm~\ref{alg:randomized_model_selection} specified to $A_{\sgdw}$ with $\eta_t\le \frac{1}{L}$. Let $u_t:=\eta^2_t + 2\eta_t\sum_{\tau =1}^{t-1} \exp(L\sum_{i=\tau+1}^{t-1} \eta_i ) \eta_\tau$ for all $t\ge 1$. Then for any $\delta \in (0,1)$ and $K \asymp\log(\frac{4}{\delta}) $,  with probability at least $1 - \delta$ over the randomness of $S$ and $\{A_{\sgdw,k}\}_{k\in [K]}$, the generalization error is upper bounded as
\[
\left|R(A_{\sgdw,k^*}(S_{k^*})) - R_S(A_{\sgdw,k^*}(S_{k^*}))\right| \lesssim G^2\sqrt{\frac{1}{N} \sum_{t=1}^T \exp\left(L\sum_{\tau=t+1}^T \eta_\tau \right) u_t} + M\sqrt{\frac{\log(1/\delta)}{N}}.
\]
\end{corollary}
\begin{remark}
For the constant learning rates $\eta_t\equiv \frac{1}{LT}$, Corollary~\ref{corol:stability_sgd_w_sm_ncvx} admits high-probability generalization bound of scale $\mathcal{O}\big(\sqrt{\frac{\log(1/\delta)}{N}}\big)$. For time decaying learning rates $\eta_t= \frac{1}{L\nu t}$ with arbitrary $\nu\ge 1$, it can be verified that the corresponding bound is of scale $\mathcal{O}\big(\sqrt{\frac{T^{1/\nu}\log(T)}{\nu N}}+\sqrt{\frac{\log(1/\delta)}{N}}\big)$.
\end{remark}

\section{Other Related Work }

The idea of using stability of a learning algorithm, namely the sensitivity of estimated model to the changes in training data, for generalization performance analysis dates back to the seventies~\citep{vapnik74theory,rogers1978finite,devroye1979distribution}. For deterministic learning algorithms, algorithmic stability has been extensively studied with a bunch of applications to establishing strong generalization and excess risk bounds for stable learning models like $k$-NN and regularized ERMs~\citep{bousquet2002stability,zhang2003leave,klochkov2021stability}. The stability theory for randomized learning algorithms was formally introduced and investigated by~\citet{elisseeff2005stability}. In a recent breakthrough work~\citep{hardt2016train}, it was shown in that the solution obtained via stochastic gradient descent is expected to be stable and generalize well for smooth convex and non-convex loss functions. For non-smooth convex losses, the stability induced generalization bounds of SGD have been established in expectation~\citep{lei2020fine} or deviation~\citep{bassily2020stability}. In the work of~\citet{kuzborskij2018data}, a set of data-dependent generalization bounds for SGD were derived based on the stability of algorithm. More broadly, generalization bounds for stable learning algorithms that converge to global minima were established in~\citet{charles2018stability,lei2021sharper}.
For non-convex sparse learning, algorithmic stability theory has been applied to derive the generalization bounds of the popularly used iterative hard thresholding (IHT) algorithm~\citep{yuan2021stability}. The uniform stability bounds on SGD have also been extensively used for designing differential privacy stochastic optimization algorithms~\citep{bassily2019private,feldman2020private}.

Bagging (or bootstrap aggregating) is one of the earliest yet most popular ensemble methods that has been widely applied to reduce the variance for unstable learning algorithms such as decision tree and neural networks~\citep{breiman1996bagging,opitz1999popular}, and sometimes stable algorithms such as SVMs~\citep{valentini2003low}. As an important variant of bagging, subbagging has been proposed to reduce the computational cost of bagging via training base models under without-replacement sampling~\citep{buhlmann2012bagging}. The stability and generalization bounds of bagging have been analyzed for both uniform~\citep{elisseeff2005stability} and non-uniform~\citep{foster2019hypothesis} averaging schemes. Unlike these prior results for bagging with averaging aggregation, our bounds are obtained based on a confidence-boosting greedy aggregation scheme which turns out to yield sharper dependence on the uniform stability parameter.

The confidence-boosting technique has long been applied for obtaining sharp high-probability excess risk bounds from the corresponding strong in-expectation bounds~\citep{shalev2010learnability,mehta2017fast}. For generic statistical learning problems, confidence-boosting has been used to convert any low-confidence learning algorithm with linear dependence on $1/\delta$ to a high-confidence algorithm with logarithmic factor $\log(1/\delta)$. For learning with exp-concave losses, a relevant ERM estimator with in-expectation fast rate of convergence was converted to a high-confidence learning algorithm with an almost identical fast rate of convergence up to a logarithmic factor on $1/\delta$~\citep{mehta2017fast}. While sharing a similar spirit of boosting the confidence, our generalization analysis is substantially more challenging than those prior excess risk analysis in terms of tightly deriving in-expectation first moment generalization bound for $L_2$-stable randomized algorithms.

\section{Conclusions}
\label{sect:conclusions}

In this paper we presented a meta framework of confidence-boosting for deriving near-optimal exponential generalization bounds for $L_2$-stable randomized learning algorithms. At a nutshell, our main results in Theorem~\ref{thrm:stability_genalization_error_highprob_bagging} and Theorem~\ref{thrm:stability_excess_risk_highprob_bagging} reveal that a carefully designed subbagging process in Algorithm~\ref{alg:randomized_model_selection} can yield high-confidence generalization and risk bounds under the weak and distribution dependent notion of $L_2$-stability. Our theory has been substantialized to SGD on both convex and non-convex losses to obtain stronger generalization bounds especially in the case of time decaying learning rates. When reduced to deterministic algorithms, the proposed method removes a logarithmic factor on sample size from the best known bounds through uniform stability. While sharper in the dependence on tail bound, our confidence-boosting results are only applicable to one of the independent runs of algorithm $A$ over $K$ disjoint training subsets of equal size with $K\asymp \log\left(\frac{1}{\delta}\right)$. It is so far not clear if these near-optimal bounds can be further extended to the full-batch setting where the generalization is with respect to the evaluation of algorithm over the entire training sample. We leave the full understanding of such an open question raised by~\citet{bousquet2020sharper} for future investigation.

\section*{Acknowledgements}

Xiao-Tong Yuan was supported in part by the National Key Research and Development Program of China under Grant No. 2018AAA0100400 and in part by Natural Science Foundation of China (NSFC) under Grant No.61876090, No.61936005 and No.U21B2049.

\newpage\clearpage

\bibliography{mybib2}

\begin{thebibliography}{40}
\providecommand{\natexlab}[1]{#1}
\providecommand{\url}[1]{\texttt{#1}}
\expandafter\ifx\csname urlstyle\endcsname\relax
  \providecommand{\doi}[1]{doi: #1}\else
  \providecommand{\doi}{doi: \begingroup \urlstyle{rm}\Url}\fi

\bibitem[Andonova et~al.(2002)Andonova, Elisseeff, Evgeniou, and
  Pontil]{andonova2002simple}
Savina Andonova, Andr{\'{e}} Elisseeff, Theodoros Evgeniou, and Massimiliano
  Pontil.
\newblock A simple algorithm for learning stable machines.
\newblock In Frank van Harmelen, editor, \emph{Proceedings of the 15th European
  Conference on Artificial Intelligence (ECAI)}, pages 513--517, Lyon, France,
  2002.

\bibitem[Bassily et~al.(2019)Bassily, Feldman, Talwar, and
  Thakurta]{bassily2019private}
Raef Bassily, Vitaly Feldman, Kunal Talwar, and Abhradeep~Guha Thakurta.
\newblock Private stochastic convex optimization with optimal rates.
\newblock In \emph{Advances in Neural Information Processing Systems
  (NeurIPS)}, pages 11279--11288, Vancouver, Canada, 2019.

\bibitem[Bassily et~al.(2020)Bassily, Feldman, Guzm{\'{a}}n, and
  Talwar]{bassily2020stability}
Raef Bassily, Vitaly Feldman, Crist{\'{o}}bal Guzm{\'{a}}n, and Kunal Talwar.
\newblock Stability of stochastic gradient descent on nonsmooth convex losses.
\newblock In \emph{Advances in Neural Information Processing Systems
  (NeurIPS)}, virtual, 2020.

\bibitem[Bengio et~al.(2017)Bengio, Goodfellow, and Courville]{bengio2017deep}
Yoshua Bengio, Ian Goodfellow, and Aaron Courville.
\newblock \emph{Deep learning}, volume~1.
\newblock MIT press Massachusetts, USA:, 2017.

\bibitem[Bousquet and Elisseeff(2002)]{bousquet2002stability}
Olivier Bousquet and Andr{\'{e}} Elisseeff.
\newblock Stability and generalization.
\newblock \emph{J. Mach. Learn. Res.}, 2:\penalty0 499--526, 2002.

\bibitem[Bousquet et~al.(2020)Bousquet, Klochkov, and
  Zhivotovskiy]{bousquet2020sharper}
Olivier Bousquet, Yegor Klochkov, and Nikita Zhivotovskiy.
\newblock Sharper bounds for uniformly stable algorithms.
\newblock In \emph{Proceedings of the Conference on Learning Theory (COLT)},
  pages 610--626, Virtual Event [Graz, Austria], 2020.

\bibitem[Breiman(1996)]{breiman1996bagging}
Leo Breiman.
\newblock Bagging predictors.
\newblock \emph{Mach. Learn.}, 24\penalty0 (2):\penalty0 123--140, 1996.

\bibitem[B{\"u}hlmann(2012)]{buhlmann2012bagging}
Peter B{\"u}hlmann.
\newblock Bagging, boosting and ensemble methods.
\newblock In \emph{Handbook of computational statistics}, pages 985--1022.
  Springer, 2012.

\bibitem[Charles and Papailiopoulos(2018)]{charles2018stability}
Zachary Charles and Dimitris~S. Papailiopoulos.
\newblock Stability and generalization of learning algorithms that converge to
  global optima.
\newblock In \emph{Proceedings of the 35th International Conference on Machine
  Learning (ICML)}, pages 744--753, Stockholmsm{\"{a}}ssan, Stockholm, Sweden,
  2018.

\bibitem[Devroye and Wagner(1979)]{devroye1979distribution}
Luc Devroye and Terry~J. Wagner.
\newblock Distribution-free inequalities for the deleted and holdout error
  estimates.
\newblock \emph{{IEEE} Trans. Inf. Theory}, 25\penalty0 (2):\penalty0 202--207,
  1979.

\bibitem[Elisseeff et~al.(2005)Elisseeff, Evgeniou, and
  Pontil]{elisseeff2005stability}
Andr{\'{e}} Elisseeff, Theodoros Evgeniou, and Massimiliano Pontil.
\newblock Stability of randomized learning algorithms.
\newblock \emph{J. Mach. Learn. Res.}, 6:\penalty0 55--79, 2005.

\bibitem[Feldman and Vondr{\'{a}}k(2018)]{feldman2018generalization}
Vitaly Feldman and Jan Vondr{\'{a}}k.
\newblock Generalization bounds for uniformly stable algorithms.
\newblock In \emph{Advances in Neural Information Processing Systems
  (NeurIPS)}, pages 9770--9780, Montr{\'{e}}al, Canada, 2018.

\bibitem[Feldman and Vondr{\'{a}}k(2019)]{feldman2019high}
Vitaly Feldman and Jan Vondr{\'{a}}k.
\newblock High probability generalization bounds for uniformly stable
  algorithms with nearly optimal rate.
\newblock In \emph{Proceedings of the Conference on Learning Theory (COLT)},
  pages 1270--1279, Phoenix, AZ, 2019.

\bibitem[Feldman et~al.(2020)Feldman, Koren, and Talwar]{feldman2020private}
Vitaly Feldman, Tomer Koren, and Kunal Talwar.
\newblock Private stochastic convex optimization: optimal rates in linear time.
\newblock In \emph{Proceedings of the 52nd Annual {ACM} {SIGACT} Symposium on
  Theory of Computing (STOC)}, pages 439--449, Chicago, IL, 2020.

\bibitem[Foster et~al.(2019)Foster, Greenberg, Kale, Luo, Mohri, and
  Sridharan]{foster2019hypothesis}
Dylan~J. Foster, Spencer Greenberg, Satyen Kale, Haipeng Luo, Mehryar Mohri,
  and Karthik Sridharan.
\newblock Hypothesis set stability and generalization.
\newblock In \emph{Advances in Neural Information Processing Systems
  (NeurIPS)}, pages 6726--6736, Vancouver, Canada, 2019.

\bibitem[Hardt et~al.(2016)Hardt, Recht, and Singer]{hardt2016train}
Moritz Hardt, Ben Recht, and Yoram Singer.
\newblock Train faster, generalize better: Stability of stochastic gradient
  descent.
\newblock In \emph{Proceedings of the 33nd International Conference on Machine
  Learning (ICML)}, pages 1225--1234, New York City, NY, 2016.

\bibitem[Harvey et~al.(2019)Harvey, Liaw, Plan, and Randhawa]{harvey2019tight}
Nicholas J.~A. Harvey, Christopher Liaw, Yaniv Plan, and Sikander Randhawa.
\newblock Tight analyses for non-smooth stochastic gradient descent.
\newblock In \emph{Proceedings of the Conference on Learning Theory (COLT)},
  pages 1579--1613, Phoenix, AZ, 2019.

\bibitem[Kale et~al.(2011)Kale, Kumar, and Vassilvitskii]{kale2011cross}
Satyen Kale, Ravi Kumar, and Sergei Vassilvitskii.
\newblock Cross-validation and mean-square stability.
\newblock In \emph{Proceedings of the Innovations in Computer Science (ICS)},
  pages 487--495, Beijing, China, 2011.

\bibitem[Klochkov and Zhivotovskiy(2021)]{klochkov2021stability}
Yegor Klochkov and Nikita Zhivotovskiy.
\newblock Stability and deviation optimal risk bounds with convergence rate
  $o(1/n)$.
\newblock In \emph{Advances in Neural Information Processing Systems
  (NeurIPS)}, pages 5065--5076, virtual, 2021.

\bibitem[Kuzborskij and Lampert(2018)]{kuzborskij2018data}
Ilja Kuzborskij and Christoph~H. Lampert.
\newblock Data-dependent stability of stochastic gradient descent.
\newblock In \emph{Proceedings of the 35th International Conference on Machine
  Learning (ICML)}, pages 2820--2829, Stockholmsm{\"{a}}ssan, Stockholm,
  Sweden, 2018.

\bibitem[Lei and Ying(2020)]{lei2020fine}
Yunwen Lei and Yiming Ying.
\newblock Fine-grained analysis of stability and generalization for stochastic
  gradient descent.
\newblock In \emph{Proceedings of the 37th International Conference on Machine
  Learning (ICML)}, pages 5809--5819, Virtual Event, 2020.

\bibitem[Lei and Ying(2021)]{lei2021sharper}
Yunwen Lei and Yiming Ying.
\newblock Sharper generalization bounds for learning with gradient-dominated
  objective functions.
\newblock In \emph{Proceedings of the 9th International Conference on Learning
  Representations (ICLR)}, Virtual Event, Austria, 2021.

\bibitem[Mason et~al.(1999)Mason, Baxter, Bartlett, and
  Frean]{mason1999boosting}
Llew Mason, Jonathan Baxter, Peter~L. Bartlett, and Marcus~R. Frean.
\newblock Boosting algorithms as gradient descent.
\newblock In \emph{Advances in Neural Information Processing Systems (NIPS)},
  pages 512--518, Denver, CA, 1999.

\bibitem[Mehta(2017)]{mehta2017fast}
Nishant~A. Mehta.
\newblock Fast rates with high probability in exp-concave statistical learning.
\newblock In \emph{Proceedings of the 20th International Conference on
  Artificial Intelligence and Statistics (AISTATS)}, pages 1085--1093, Fort
  Lauderdale, FL, 2017.

\bibitem[Mukherjee et~al.(2006)Mukherjee, Niyogi, Poggio, and
  Rifkin]{mukherjee2006learning}
Sayan Mukherjee, Partha Niyogi, Tomaso~A. Poggio, and Ryan~M. Rifkin.
\newblock Learning theory: stability is sufficient for generalization and
  necessary and sufficient for consistency of empirical risk minimization.
\newblock \emph{Adv. Comput. Math.}, 25\penalty0 (1-3):\penalty0 161--193,
  2006.

\bibitem[Nagaraj et~al.(2019)Nagaraj, Jain, and Netrapalli]{nagaraj19a}
Dheeraj Nagaraj, Prateek Jain, and Praneeth Netrapalli.
\newblock {SGD} without replacement: Sharper rates for general smooth convex
  functions.
\newblock In \emph{Proceedings of the 36th International Conference on Machine
  Learning (ICML)}, pages 4703--4711, Long Beach, CA, 2019.

\bibitem[Opitz and Maclin(1999)]{opitz1999popular}
David~W. Opitz and Richard Maclin.
\newblock Popular ensemble methods: An empirical study.
\newblock \emph{J. Artif. Intell. Res.}, 11:\penalty0 169--198, 1999.

\bibitem[Rakhlin et~al.(2012)Rakhlin, Shamir, and Sridharan]{rakhlin2012making}
Alexander Rakhlin, Ohad Shamir, and Karthik Sridharan.
\newblock Making gradient descent optimal for strongly convex stochastic
  optimization.
\newblock In \emph{Proceedings of the 29th International Conference on Machine
  Learning (ICML)}, Edinburgh, Scotland, UK, 2012.

\bibitem[Rogers and Wagner(1978)]{rogers1978finite}
William~H Rogers and Terry~J Wagner.
\newblock A finite sample distribution-free performance bound for local
  discrimination rules.
\newblock \emph{The Annals of Statistics}, pages 506--514, 1978.

\bibitem[Schapire(1990)]{schapire1990strength}
Robert~E. Schapire.
\newblock The strength of weak learnability.
\newblock \emph{Mach. Learn.}, 5:\penalty0 197--227, 1990.

\bibitem[Schmidt et~al.(2011)Schmidt, Roux, and Bach]{schmidt2011convergence}
Mark Schmidt, Nicolas~Le Roux, and Francis~R. Bach.
\newblock Convergence rates of inexact proximal-gradient methods for convex
  optimization.
\newblock In \emph{Advances in Neural Information Processing Systems (NIPS)},
  pages 1458--1466, Granada, Spain, 2011.

\bibitem[Shalev{-}Shwartz et~al.(2009)Shalev{-}Shwartz, Shamir, Srebro, and
  Sridharan]{shalev2009stochastic}
Shai Shalev{-}Shwartz, Ohad Shamir, Nathan Srebro, and Karthik Sridharan.
\newblock Stochastic convex optimization.
\newblock In \emph{Proceedings of the 22nd Conference on Learning Theory
  (COLT)}, Montreal, Canada, 2009.

\bibitem[Shalev{-}Shwartz et~al.(2010)Shalev{-}Shwartz, Shamir, Srebro, and
  Sridharan]{shalev2010learnability}
Shai Shalev{-}Shwartz, Ohad Shamir, Nathan Srebro, and Karthik Sridharan.
\newblock Learnability, stability and uniform convergence.
\newblock \emph{J. Mach. Learn. Res.}, 11:\penalty0 2635--2670, 2010.

\bibitem[Shamir(2016)]{shamir2016without}
Ohad Shamir.
\newblock Without-replacement sampling for stochastic gradient methods.
\newblock In \emph{Advances in Neural Information Processing Systems (NIPS)},
  pages 46--54, Barcelona, Spain, 2016.

\bibitem[Shamir and Zhang(2013)]{shamir2013stochastic}
Ohad Shamir and Tong Zhang.
\newblock Stochastic gradient descent for non-smooth optimization: Convergence
  results and optimal averaging schemes.
\newblock In \emph{Proceedings of the 30th International Conference on Machine
  Learning (ICML)}, pages 71--79, Atlanta, GA, 2013.

\bibitem[Valentini and Dietterich(2003)]{valentini2003low}
Giorgio Valentini and Thomas~G. Dietterich.
\newblock Low bias bagged support vector machines.
\newblock In \emph{Proceedings of the Twentieth International Conference on
  Machine Learning (ICML)}, pages 752--759, Washington, DC, 2003.

\bibitem[Vapnik and Chervonenkis(1974)]{vapnik74theory}
V.~N. Vapnik and A.~Ya. Chervonenkis.
\newblock \emph{Theory of Pattern Recognition [in Russian]}.
\newblock Nauka, 1974.

\bibitem[Yuan and Li(2021)]{yuan2021stability}
Xiaotong Yuan and Ping Li.
\newblock Stability and risk bounds of iterative hard thresholding.
\newblock In \emph{Proceedings of the 24th International Conference on
  Artificial Intelligence and Statistics (AISTATS)}, pages 1702--1710, Virtual
  Event, 2021.

\bibitem[Zhang(2003)]{zhang2003leave}
Tong Zhang.
\newblock Leave-one-out bounds for kernel methods.
\newblock \emph{Neural Comput.}, 15\penalty0 (6):\penalty0 1397--1437, 2003.

\bibitem[Zhou et~al.(2022)Zhou, Liang, and Zhang]{zhou2019understanding}
Yi~Zhou, Yingbin Liang, and Huishuai Zhang.
\newblock Understanding generalization error of {SGD} in nonconvex
  optimization.
\newblock \emph{Mach. Learn.}, 111\penalty0 (1):\penalty0 345--375, 2022.

\end{thebibliography}
\bibliographystyle{plainnat}


\newpage
\appendix

\section{Auxiliary Lemmas}
\label{apdsect:auxiliary lemmas}

We need the following lemma from~\citet{hardt2016train} which shows that SGD iteration is non-expansive on convex and smooth loss.
\begin{lemma}[\citet{hardt2016train}]\label{lemma:non_expansion_convex}
Assume that $f$ is convex and $L$-smooth. Then for any $w,w' \in \mathcal{W}$ and $\alpha \le 2/L$, we have the following bound holds
\[
\|w - \alpha \nabla f(w) - (w' - \alpha \nabla f(w'))\| \le \|w - w'\|.
\]
\end{lemma}

The following lemma, which can be proved by induction~\citep[see, e.g., ][]{schmidt2011convergence}, will be used to prove the main results in Section~\ref{sect:implication_sgd}.
\begin{lemma}\label{lemma:key_sequence_bound}
Assume that the nonnegative sequence $\{u_\tau\}_{\tau\ge1}$ satisfies the following recursion for all $t\ge 1$:
\[
u^2_t \le S_t + \sum_{\tau=1}^t \alpha_\tau u_\tau,
\]
with $\{S_\tau\}_{\tau\ge1}$ an increasing sequence, $S_0\ge u_0^2$ and $\alpha_\tau\ge 0$ for all $\tau$. Then, the following inequality holds for all $t\ge 1$:
\[
u_t \le \sqrt{S_t} + \sum_{\tau=1}^t \alpha_\tau.
\]
\end{lemma}

For analyzing SGD with convex and non-smooth loss functions, we need the following lemma by~\citet[Lemma 3.1]{bassily2020stability} that quantifies the deviation between the online gradient descent trajectories.

\begin{lemma}[\citet{bassily2020stability}]\label{lemma:expansion_convex_nonsm}
Consider the two sequences $\{w_t\}_{t\ge 0}$ and $\{w'_t\}_{t\ge 0}$ generated according to the following recursions respectively over the convex and $G$-Lipschitz objectives  $\{f_t\}_{t\ge 0}$ and $\{f'_t\}_{t\ge 0}$ from $w_0 = w'_0$:
\[
\begin{aligned}
w_t =& \Pi_{\mathcal{W}}\left( w_{t-1} - \eta_t \nabla f_{t-1}(w_{t-1}) \right) \\
w'_t =& \Pi_{\mathcal{W}}\left( w'_{t-1} - \eta_t \nabla f'_{t-1}(w'_{t-1}) \right) .
\end{aligned}
\]
Let $t_0:=\inf\{t: f_t\neq f'_t\}$ and $\beta_t:= \mathbf{1}_{\left\{f_t\neq f'_t\right\}}$. Then for any $T\ge 1$,
\[
\|w_T - w'_T\| \le 2G\sqrt{\sum_{t=t_0}^{T-1}\eta_t^2} + 4G\sum_{t=t_0+1}^{T-1}\eta_t \beta_t.
\]
\end{lemma}

\newpage

\section{Proofs for the Results in Section~\ref{sect:general_results}}
\label{apdsect:proof_main_results}
In this section, we present the technical proofs for the main results stated in Section~\ref{sect:general_results}.
\subsection{Proof of Lemma~\ref{lemma:stability_genalization_first_order_moment}}
\label{apdsect:proof_lemma_stability_generalization_error_average}

We need the following lemma essentially from~\citet{bousquet2020sharper} that provides a first moment bound for the sum of random functions.
\begin{lemma}\label{lemma:bousquet_second_moment}
Let $S=\{Z_1, Z_2, ...,Z_N\}$ be a set of i.i.d. random variables valued in $\mathcal{Z}$. Let $g_1,...,g_N$ be a set of measurable functions $g_i: \mathcal{Z}^N \mapsto \mathbb{R}$ that satisfy $\|g_i(S)\|_2 \le M$ and $\mathbb{E}_{Z_i}[g_i(S)] =0$ for all $i\in[N]$. Then we have
\[
\left\|\sum_{i=1}^N g_i(S)\right\|_2 \le \sqrt{\sum_{i\neq j} \left\|g_i(S) - g_i(S^{(j)})\right\|_2^2 } + M \sqrt{N}.
\]
\end{lemma}
\begin{proof}
We reproduce the proof in view of the argument in~\citet[Section 5]{bousquet2020sharper} showing that $\{g_i\}$ are weakly correlated. For any $i\neq j$, since $\mathbb{E}_{Z_i}[g_i(S)]=0$ and $\mathbb{E}_{Z_j}[g_j(S)] = 0$, we can verify that
\[
\begin{aligned}
\mathbb{E}_S\left[g_i(S^{(j)})g_j(S)\right] =& \mathbb{E}_{S\setminus Z_j}\left[\mathbb{E}_{Z_j}\left[g_i(S^{(j)})g_j(S)\mid S\setminus Z_j \right]\right] \\
=& \mathbb{E}_{S\setminus Z_j}\left[g_i(S^{(j)})\mathbb{E}_{Z_j}\left[g_j(S)\mid S\setminus Z_j \right]\right]=0,
\end{aligned}
\]
where we have used the independence of the elements in $S \cup \{Z'_j\}$. Similarly, we can show that
\[
\mathbb{E}_S\left[g_i(S)g_j(S^{(i)})\right] =  \mathbb{E}_S \left[g_i(S^{(j)})g_j(S^{(i)})\right]=0.
\]
Then it follows that for any $i\neq j$,
\[
\begin{aligned}
\left|\mathbb{E}_S\left[g_i(S) g_j(S)\right]\right| =& \left|\mathbb{E}_{S,S^{(i)},S^{(j)}}\left[g_i(S) g_j(S)\right]\right|\\
=& \left|\mathbb{E}_{S,S^{(i)},S^{(j)}}\left[(g_i(S) - g_i(S^{(j)})) (g_j(S) - g_j(S^{(i)}))\right]\right| \\
\le& \mathbb{E}_{S,S^{(i)},S^{(j)}} \left[\left|(g_i(S) - g_i(S^{(j)})) (g_j(S) - g_j(S^{(i)}))\right|\right].
\end{aligned}
\]
Based on the above bound we can further show that
\[
\begin{aligned}
\left\|\sum_{i=1}^N g_i(S)\right\|_2 =& \sqrt{\mathbb{E}_{S}\left[\left(\sum_{i=1}^N g_i(S)\right)^2 \right]} =  \sqrt{\sum_{i\neq j}\mathbb{E}_{S}\left[g_i(S) g_j(S) \right] + \sum_{i=1}^N \mathbb{E}_S[g^2_i(S)]} \\
\le&  \sqrt{\sum_{i\neq j} \mathbb{E}_{S,S^{(i)},S^{(j)}}\left[\left|(g_i(S) - g_i(S^{(j)})) (g_j(S) - g_j(S^{(i)}))\right|\right] } + M \sqrt{N} \\
\le& \sqrt{\frac{1}{2}\sum_{i\neq j} \mathbb{E}_{S,S^{(i)},S^{(j)}}\left[\left(g_i(S) - g_i(S^{(j)})\right)^2 +\left(g_j(S) - g_j(S^{(i)})\right)^2\right] } + M \sqrt{N}\\
=& \sqrt{\sum_{i\neq j} \mathbb{E}_{S,S^{(j)}}\left[\left(g_i(S) - g_i(S^{(j)})\right)^2 \right] } + M \sqrt{N} \\
=& \sqrt{\sum_{i\neq j} \left\|g_i(S) - g_i(S^{(j)})\right\|_2^2 } + M \sqrt{N}.
\end{aligned}
\]
This proves the desired bound.
\end{proof}
Now we are ready to prove the result in Lemma~\ref{lemma:stability_genalization_first_order_moment}.
\begin{proof}[Proof of Lemma~\ref{lemma:stability_genalization_first_order_moment}]
Let us consider $h_i(S):= R(A(S)) - \ell(A(S);Z_i)$ and $g_i(S)= h_i(S) - \mathbb{E}_{Z_i}[h_i(S)]$ for $i\in [N]$. Then by assumption we have
\[
\mathbb{E}_{Z_i} [g_i(S)] = 0, \ \ \ \|g_i(S)\|_2 \le \|h_i(S)\|_2 \le 2M.
\]
For any $j\neq i$, since the loss is non-negative, it can be verified that
\[
|g_i(S) - g_i(S^{(j)})| \le \max\left\{\left|h_i(S) - h_i(S^{(j)})\right|, \left|\mathbb{E}_{Z_i}[h_i(S) - h_i(S^{(j)})]\right| \right\},
\]
which readily implies
\[
\left\|g_i(S) - g_i(S^{(j)})\right\|_2 \le \left\| h_i(S) - h_i(S^{(j)}) \right\|_2 \le 2 \gamma_{\ms,N},
\]
where in the last inequality we have used the $L_2$-stability conditions on $A$. Then invoking Lemma~\ref{lemma:bousquet_second_moment} to $\{g_i\}$ yields
\begin{equation}\label{inequat:second_moment_key1}
\left\|\sum_{i=1}^N g_i(S)\right\|_2 \le \sqrt{\sum_{i\neq j} \left\|g_i(S) - g_i(S^{(j)})\right\|_2^2 } + 2M \sqrt{N} \le 2N\gamma_{\ms,N} + 2M \sqrt{N}.
\end{equation}
Further, it can be verified that
\begin{equation}\label{inequat:second_moment_key2}
\begin{aligned}
 \left\|\sum_{i=1}^N\mathbb{E}_{Z_i} [ h_i(S) ]\right\|_2 
=& \left\|\sum_{i=1}^N\mathbb{E}_{Z_i} [ R(A(S)) - \ell(A(S);Z_i)] \right\|_2 \\
=& \left\|\sum_{i=1}^N\mathbb{E}_{Z_i,Z'_i} [\ell(A(S);Z'_i) - \ell(A(S);Z_i)] \right\|_2 \\
=& \left\|\sum_{i=1}^N\mathbb{E}_{Z_i} \left[ \mathbb{E}_{Z'_i} [\ell(A(S);Z'_i)] - \mathbb{E}_{Z'_i}[\ell(A(S^{(i)});Z'_i)]\right] \right\|_2 \\
\le& \sum_{i=1}^N \left\|\ell(A(S);Z'_i) - \ell(A(S^{(i)});Z'_i) \right\|_2 \le N\gamma_{\ms,N}.
\end{aligned}
\end{equation}
By combining~\eqref{inequat:second_moment_key1} and~\eqref{inequat:second_moment_key2} we obtain
\[
\begin{aligned}
&\mathbb{E}_{A,S}\left[\left|R(A(S)) - R_S(A(S))\right| \right] \\
=& \left\|R(A(S)) - R_S(A(S))\right\|_1 \\
\le& \left\|R(A(S)) - R_S(A(S))\right\|_2 \\
=& \frac{1}{N} \left\|\sum_{i=1}^N \left( g_i(S)+\mathbb{E}_{Z_i} [ h_i(S) ]\right)\right\|_2\\
\le&  \frac{1}{N} \left\|\sum_{i=1}^N g_i(S)\right\|_2 +\frac{1}{N} \left\|\sum_{i=1}^N \mathbb{E}_{Z_i} [ h_i(S) ]\right\|_2 \\
\le& 3 \gamma_{\ms,N} + \frac{2M}{\sqrt{N}}.
\end{aligned}
\]
The proof is completed.
\end{proof}

\subsection{Proof of Theorem~\ref{thrm:stability_genalization_error_highprob_bagging}}
\label{apdsect:proof_thrm_stability_genalization_error_highprob_bagging}
We first establish the following intermediate result that captures the effects of subbagging on randomized algorithms: it basically tells that with $K\asymp\log(\frac{1}{\delta})$, at least one of the solutions generated by subbagging generalizes well with high probability.

\begin{lemma}\label{thrm:stability_genalization_error_highprob}
Suppose that a randomized learning algorithm $A: \mathcal{Z}^N \mapsto \mathcal{W}$ has $L_2$-stability by $\gamma_{\ms,N}$.  Assume that the loss function satisfies $\|\ell(A(S);Z)\|_2 \le M$.  Then for any $\alpha,\delta \in (0,1)$ and $K \ge\frac{\log(2/\delta)}{1-\alpha} $, with probability at least $1 - \delta$ over the randomness of $\{(A_k, S_k)\}_{k\in [K]}$, the sequence $\{A_k(S_k)\}_{k \in [K]}$ generated by Algorithm~\ref{alg:randomized_model_selection} satisfies
\[
\min_{k\in [K]} \left|R( A_k(S_k)) - R_{S_k}( A_k(S_k))\right| \lesssim \frac{1}{\alpha}\left( \gamma_{\ms,\frac{N}{K}}+ M\sqrt{\frac{K}{N}}\right).
\]
\end{lemma}
\begin{proof}
From Lemma~\ref{lemma:stability_genalization_first_order_moment} we have that over randomized algorithm $A$ and data $\tilde S$ with $|\tilde S|=\frac{N}{K}$,
\[
\mathbb{E}_{A,\tilde S} \left[|R(A(\tilde S)) - R_{\tilde S} (A(\tilde S))|\right] \le 3G\gamma_{\ms,\frac{N}{K}} + 2M\sqrt{\frac{K}{N}}.
\]
Since $\{A_k,S_k\}_{k\in[K]}$ are independent to each other, by Markov inequality we know that
\[
\mathbb{P}_{\{A_k,S_k\}} \left(\min_{k\in [K]} \left|R( A_k(S_k)) - R_{S_k}( A_k(S_k))\right| \ge \frac{1}{\alpha}\left(3\gamma_{\ms,\frac{N}{K}} + 2M\sqrt{\frac{K}{N}}\right)\right) \le \alpha^K \le \delta,
\]
which implies the desired bound.
\end{proof}

With Lemma~\ref{thrm:stability_genalization_error_highprob} in place, we can proceed to prove the main result in Theorem~\ref{thrm:stability_genalization_error_highprob_bagging}.
\begin{proof}[Proof of Theorem~\ref{thrm:stability_genalization_error_highprob_bagging}]
Let us consider the following three events:
\[
\begin{aligned}
\mathcal{E}:=& \left\{ \left|R(A_{k^*}(S_{k^*})) - R_S(A_{k^*}(S_{k^*}))\right| \lesssim \frac{1}{\alpha K}\left(\gamma_{\ms,\frac{N}{K}} + M\sqrt{\frac{K}{N}}\right)  + M\sqrt{\frac{K\log(K/\delta)}{(K-1)N}} \right\}, \\
\mathcal{E}_1:=& \left\{ \max_{k\in [K]} |R(A_k(S_k)) - R_{S\setminus S_k} (A_k(S_k)) | \lesssim M\sqrt{\frac{K\log(K/\delta)}{(K-1)N}}\right\}, \\
\mathcal{E}_2:=& \left\{ \min_{k\in [K]} \left|R(A_k(S_k)) - R_{S_k}(A_k(S_k))\right| \lesssim \frac{1}{\alpha}\left(\gamma_{\ms,\frac{N}{K}} + M\sqrt{\frac{K}{N}}\right) \right\}.
\end{aligned}
\]
We can show that $\mathcal{E} \supseteq \mathcal{E}_1 \cap \mathcal{E}_2$. Indeed, suppose that $\mathcal{E}_1$ and $\mathcal{E}_2$ simultaneously occur. Consequently the following inequality is valid:
\[
\begin{aligned}
&\left|R(A_{k^*}(S_{k^*})) - R_S(A_{k^*}(S_{k^*}))\right| \\
=& \left|R(A_{k^*}(S_{k^*})) - \frac{1}{K}R_{S_{k^*}}(A_{k^*}(S_{k^*})) - \frac{K-1}{K}R_{S\setminus S_{k^*}}(A_{k^*}(S_{k^*}))\right| \\
\le& \frac{1}{K} \left|R(A_{k^*}(S_{k^*})) - R_{S_{k^*}}(A_{k^*}(S_{k^*}))\right| + \frac{K-1}{K} \left|R(A_{k^*}(S_{k^*})) - R_{S\setminus S_{k^*}}(A_{k^*}(S_{k^*}))\right| \\
\le& \frac{1}{K}\left|R_{S\setminus S_{k^*}}(A_{k^*}(S_{k^*})) - R_{S_{k^*}}(A_{k^*}(S_{k^*}))\right| + \left|R(A_{k^*}(S_{k^*})) - R_{S\setminus S_{k^*}}(A_{k^*}(S_{k^*}))\right| \\
\overset{\zeta_1}{=}& \frac{1}{K}\min_{k\in [K]} \left|R_{S\setminus S_{k}}(A_{k}(S_k)) - R_{S_k}(A_{k}(S_k))\right| + \left|R(A_{k^*}(S_{k^*})) - R_{S\setminus S_{k^*}}(A_{k^*}(S_{k^*}))\right| \\
=& \frac{1}{K}\min_{k\in [K]} \left|R_{S\setminus S_{k}}(A_{k}(S_k)) - R(A_{k}(S_k)) + R(A_{k}(S_k)) - R_{S_k}(A_{k}(S_k))\right| \\
&+ \left|R(A_{k^*}(S_{k^*})) - R_{S\setminus S_{k^*}}(A_{k^*}(S_{k^*}))\right| \\
\le& \frac{1}{K}\min_{k\in [K]} \left|R(A_{k}(S_k)) - R_{S_k}(A_{k}(S_k))\right| + \frac{1}{K}\max_{k\in [K]}\left|R_{S\setminus S_k}(A_{k}(S_k)) - R(A_{k}(S_k))\right| \\
&+ \left|R(A_{k^*}(S_{k^*})) - R_{S\setminus S_{k^*}}(A_{k^*}(S_{k^*}))\right|\\
\le& \frac{1}{K}\min_{k\in [K]} \left|R(A_{k}(S_k)) - R_{S_k}(A_{k}(S_k))\right| + \frac{K+1}{K}\max_{k\in [K]}\left|R(A_{k}(S_k)) - R_{S\setminus S_k}(A_{k}(S_k))\right| \\
\overset{\zeta_2}{\lesssim} & \frac{1}{\alpha K}\left(\gamma_{\ms,\frac{N}{K}} + M\sqrt{\frac{K}{N}}\right) + M \sqrt{\frac{K\log(K/\delta)}{(K-1)N}},
\end{aligned}
\]
where in ``$\zeta_1$'' we have used the definition of $k^*$, and ``$\zeta_2$'' follows from $\mathcal{E}_1, \mathcal{E}_2$. With leading terms preserved in the above we can see that $\mathcal{E}$ occurs.

Next we can show that $\mathbb{P}_{S,\{A_k\}} (\overline{\mathcal{E}_1}) \le \frac{\delta}{2}$. Toward this end, let us consider the following events for all $k\in [K]$:
\[
\mathcal{E}^k_{1}:= \left\{ |R(A_k(S_k)) - R_{S\setminus S_k} (A_k(S_k)) | \lesssim M\sqrt{\frac{K\log(K/\delta)}{(K-1)N}}\right\}.
\]
Clearly, it is true that $\mathcal{E}_1 =\bigcap_{k=1}^K \mathcal{E}^k_{1}$. It is sufficient to prove that $\mathbb{P}_{S, A_k} \left( \overline{\mathcal{E}^k_1}\right) \le \frac{\delta}{2K}$ holds for each $k\in [K]$. Indeed, consider the random indication function $\beta (S, A_k):= \mathbf{1}_{\overline{\mathcal{E}^k_1}}$ associated with the event $\overline{\mathcal{E}^k_1}$. Then we have the following holds for each $k\in [K]$:
\[
\begin{aligned}
&\mathbb{P}_{S, A_k} \left( \overline{\mathcal{E}^k_1}\right) \\
=& \mathbb{E}_{S,A_k}\left[\beta(S,A_k)\right] \\
=& \mathbb{E}_{A_k,S_k} \left[\mathbb{E}_{S\setminus S_k \mid A_k,S_k} \left[ \beta(S,A_k) \mid A_k,S_k \right]\right] \\
\overset{\zeta_1}{=}& \mathbb{E}_{A_k,S_k} \left[\mathbb{E}_{S\setminus S_k} \left[ \beta(S,A_k) \mid A_k,S_k \right]\right] \\
=& \mathbb{E}_{A_k,S_k} \left[ \mathbb{P}_{S\setminus S_k}\left(|R(A_k(S_k)) - R_{S\setminus S_k} (A_k(S_k)) | \gtrsim M \sqrt{\frac{K\log(K/\delta)}{(K-1)N}} \right) \mid  A_k,S_k\right] \\
\overset{\zeta_2}{\le}& \mathbb{E}_{A_k,S_k} \left[\frac{\delta}{2K} \mid A_k,S_k\right] = \frac{\delta}{2K},
\end{aligned}
\]
where in ``$\zeta_1$'' we have used the independence between $\{A_k,S_k\}$ and $S\setminus S_k$, and ``$\zeta_2$'' is due to Hoeffding's inequality conditioned on $\{A_k,S_k\}$, keeping in mind that $A_k(S_k)$ is independent on the data set $S\setminus S_k$ of size $(1-1/K)N$. It follows by union probability that
\[
\mathbb{P}_{S, \{A_k\}} \left( \overline{\mathcal{E}_1}\right) = \mathbb{P}_{S, \{A_k\}} \left( \bigcup_{k=1}^K \overline{\mathcal{E}^k_1}\right) \le \sum_{k=1}^K \mathbb{P}_{S, A_k} \left( \overline{\mathcal{E}^k_1}\right) \le \frac{\delta}{2}.
\]
Further, from the part(b) of Lemma~\ref{thrm:stability_genalization_error_highprob} we have $\mathbb{P}_{S,\{A_k\}} (\bar{\mathcal{E}_2}) \le \frac{\delta}{2}$. Combining this and the preceding bound yields
\[
\mathbb{P}_{S, \{A_k\}} \left(\mathcal{E}\right) \ge \mathbb{P}_{S,\{A_k\}}\left(\mathcal{E}_1\cap \mathcal{E}_2 \right) \ge 1- \mathbb{P}_{S, \{A_k\}} \left(\bar{\mathcal{E}_1}\right) - \mathbb{P}_{S,\{A_k\}} (\bar{\mathcal{E}_2}) \ge 1-\frac{\delta}{2} - \frac{\delta}{2} = 1-\delta.
\]
This implies the desired result in part(b) as $K/(K-1)\le 2$ for $K\ge 2$.
\end{proof}

\subsection{Proof of Theorem~\ref{thrm:stability_excess_risk_highprob_bagging}}
\label{apdsect:proof_thrm_stability_excess_risk_highprob_bagging}

We first present the following simple lemma about the in-expectation risk bounds of a randomized algorithm which will be used in our analysis.
\begin{lemma}\label{lemma:stability_excessrisk}
Suppose that a randomized learning algorithm $A: \mathcal{Z}^N \mapsto \mathcal{W}$ has $L_2$-stability by $\gamma_{\ms,N}$. Then,
\[
\mathbb{E}_{A,S} \left[ R(A(S)) - R(w^*) \right] \le \gamma_{\ms,N} + \Delta_{\opt}.
\]
\end{lemma}
\begin{proof}
It can be verified that
\begin{equation}\label{inequat:risk_key1}
\begin{aligned}
&\left| \mathbb{E}_{A,S} \left[R(A(S)) - R_S(A(S)) \right] \right|  \\
=& \frac{1}{N} \left| \sum_{i=1}^N\mathbb{E}_{A,S} [ R(A(S)) - \ell(A(S);Z_i)] \right| \\
=& \frac{1}{N} \left|\sum_{i=1}^N\mathbb{E}_{A, S,Z'_i} [\ell(A(S);Z'_i) - \ell(A(S);Z_i)] \right| \\
=& \frac{1}{N} \left|\sum_{i=1}^N\mathbb{E}_{A, S,Z'_i} [\ell(A(S);Z'_i) - \ell(A(S^{(i)});Z'_i)] \right| \\
\le& \frac{1}{N} \sum_{i=1}^N \left\|\ell(A(S);Z'_i) - \ell(A(S^{(i)});Z'_i) \right\|_2 \le \gamma_{\ms,N}.
\end{aligned}
\end{equation}
By standard risk decomposition we can show that
\[
\begin{aligned}
&\mathbb{E}_{A,S} \left[ R(A(S)) - R(w^*) \right] \\
=&  \mathbb{E}_{A,S} \left[R(A(S)) - R_S(A(S)) + R_S(A(S)) - R_S(w^*) + R_S(w^*) - R(w^*) \right] \\
\le& \left| \mathbb{E}_{A,S} \left[R(A(S)) - R_S(A(S)) \right] \right| + \Delta_{\opt} \\
\le& \gamma_{\ms, N} + \Delta_{\opt},
\end{aligned}
\]
where in the last inequality we have used~\eqref{inequat:risk_key1}.
\end{proof}

\begin{lemma}\label{thrm:stability_excess_risk_highprob}
Suppose that a randomized learning algorithm $A: \mathcal{Z}^N \mapsto \mathcal{W}$ has $L_2$-stability by $\gamma_{\ms,N}$. Then for any $\alpha,\delta \in (0,1)$ and $K \ge\frac{\log(2/\delta)}{1-\alpha} $, with probability at least $1 - \delta$ over the randomness of $\{(A_k, S_k)\}_{k\in [K]}$, the sequence $\{A_k(S_k)\}_{k \in [K]}$ generated by Algorithm~\ref{alg:randomized_model_selection} with the modified output satisfies
\[
\min_{k\in [K]} R( A_k(S_k)) - R(w^*) \lesssim \frac{1}{\alpha}\left( \gamma_{\ms,\frac{N}{K}} + \Delta_{\opt} \right).
\]
\end{lemma}
\begin{proof}
Recall the modified output $A_{k^*}(S_{k^*})$ where $k^* = \argmin_{k\in [K]} R_{S\setminus S_k }(A_k(S_k))$. From Lemma~\ref{lemma:stability_excessrisk} we have that over randomized algorithm $A$ and data $\tilde S$ with $|\tilde S|=\frac{N}{K}$,
\[
\mathbb{E}_{A,\tilde S} \left[R(A(\tilde S)) - R(w^*) \right] \le \gamma_{\ms,\frac{N}{K}} + \Delta_{\opt}.
\]
Since $\{A_k,S_k\}_{k\in[K]}$ are independent to each other, by Markov inequality we know that
\[
\mathbb{P}_{\{A_k,S_k\}} \left(\min_{k\in [K]} R( A_k(S_k)) - R( w^*) \ge \frac{1}{\alpha}\left( \gamma_{\ms,\frac{N}{K}} + \Delta_{\opt} \right)\right) \le \alpha^K \le \delta,
\]
which implies the desired bound in part(b).
\end{proof}

Next we proceed to prove the main result in Theorem~\ref{thrm:stability_excess_risk_highprob_bagging}.
\begin{proof}[Proof of Theorem~\ref{thrm:stability_excess_risk_highprob_bagging}]
Let us consider the following three events:
\[
\begin{aligned}
\mathcal{E}:=& \left\{ \left|R(A_{k^*}(S_{k^*})) - R_S(A_{k^*}(S_{k^*}))\right| \lesssim \frac{1}{\alpha}\left(\gamma_{\ms,\frac{N}{K}} + \Delta_{\opt} \right)  + M\sqrt{\frac{K\log(K/\delta)}{(K-1)N}} \right\}, \\
\mathcal{E}_1:=& \left\{ \max_{k\in [K]} |R(A_k(S_k)) - R_{S\setminus S_k} (A_k(S_k)) | \lesssim M\sqrt{\frac{K\log(K/\delta)}{(K-1)N}}\right\}, \\
\mathcal{E}_2:=& \left\{ \min_{k\in [K]} R(A_k(S_k)) - R(w^*) \lesssim \frac{1}{\alpha}\left(\gamma_{\ms,\frac{N}{K}} + \Delta_{\opt} \right) \right\}.
\end{aligned}
\]
Similarly, we show that $\mathcal{E} \supseteq \mathcal{E}_1 \cap \mathcal{E}_2$. Indeed, suppose that $\mathcal{E}_1$ and $\mathcal{E}_2$ simultaneously occur. Consequently the following inequality is valid:
\[
\begin{aligned}
&R(A_{k^*}(S_{k^*})) - R(w^*) \\
=& R(A_{k^*}(S_{k^*})) - R_{S\setminus S_{k^*}}(A_{k^*}(S_{k^*})) + R_{S\setminus S_{k^*}}(A_{k^*}(S_{k^*})) - R(w^*) \\
\overset{\zeta_1}{=}& R(A_{k^*}(S_{k^*})) - R_{S\setminus S_{k^*}}(A_{k^*}(S_{k^*})) + \min_{k\in[K]}R_{S\setminus S_{k}}(A_{k}(S_{k})) - R(w^*) \\
=& R(A_{k^*}(S_{k^*})) - R_{S\setminus S_{k^*}}(A_{k^*}(S_{k^*})) + \min_{k\in[K]} \left\{R_{S\setminus S_{k}}(A_{k}(S_{k})) - R(A_{k}(S_{k})) + R(A_{k}(S_{k})) - R(w^*)\right\} \\
\le& \min_{k\in[K]} (R(A_k(S_k)) - R(w^*)) + 2\max_{k\in [K]}\left|R_{S\setminus S_k}(A_{k}(S_k)) - R(A_{k}(S_k))\right| \\
\overset{\zeta_2}{\lesssim} & \frac{1}{\alpha}\left(\gamma_{\ms,\frac{N}{K}} + \Delta_{\opt}\right) + M \sqrt{\frac{K\log(K/\delta)}{(K-1)N}},
\end{aligned}
\]
where in ``$\zeta_1$'' we have used the definition of $k^*$, and ``$\zeta_2$'' follows from $\mathcal{E}_1, \mathcal{E}_2$. With leading terms preserved in the above we can see that $\mathcal{E}$ occurs.

Based on the identical proof argument as that of the part(b) of Theorem~\ref{thrm:stability_genalization_error_highprob_bagging} we can show that $\mathbb{P}_{S,\{A_k\}} (\overline{\mathcal{E}_1}) \le \frac{\delta}{2}$.
Further, from the part(b) of Lemma~\ref{thrm:stability_excess_risk_highprob} we have $\mathbb{P}_{S,\{A_k\}} (\bar{\mathcal{E}_2}) \le \frac{\delta}{2}$. Combining this and the preceding bound yields
\[
\mathbb{P}_{S, \{A_k\}} \left(\mathcal{E}\right) \ge \mathbb{P}_{S,\{A_k\}}\left(\mathcal{E}_1\cap \mathcal{E}_2 \right) \ge 1- \mathbb{P}_{S, \{A_k\}} \left(\bar{\mathcal{E}_1}\right) - \mathbb{P}_{S,\{A_k\}} (\bar{\mathcal{E}_2}) \ge 1-\frac{\delta}{2} - \frac{\delta}{2} = 1-\delta.
\]
This implies the desired result in part(b) as $K/(K-1)\le 2$ for $K\ge 2$.
\end{proof}

\section{Proofs for the Results in Section~\ref{sect:implication_sgd}}
\label{apdsect:proof_sgd_results}
In this section, we present the technical proofs for the main results stated in Section~\ref{sect:implication_sgd}.

\subsection{Proof of Corollary~\ref{corol:stability_sgd_w_convex_smooth}}
\label{apdsect:proof_lemma:uniform_stability_sgd_w_sm_cvx}

We begin with presenting and proving the following lemma that gives the $L_2$-stability bounds for $A_{\sgdw}$ on convex and smooth loss functions such as logistic loss.

\begin{lemma}\label{lemma:uniform_stability_sgd_w_sm_cvx}
Suppose that the loss function is $\ell(\cdot; \cdot)$ is convex, $G$-Lipschitz and $L$-smooth with respect to its first argument. Assume that $\eta_t\le 2/L$ for all $t\ge 1$. Then $A_{\sgdw}$ has $L_2$-stability by
\[
\gamma_{\ms,N}  = G^2\sqrt{\frac{40}{N}\left(\sum_{t=1}^T \eta^2_t + \frac{1}{N} \left(\sum_{t=1}^T \eta_t\right)^2\right)}.
\]
\end{lemma}
\begin{proof}
For any pair of data sets $S,S'$ that differ in a single element, let us define the sequences $\{w_t\}_{t\in[T]}$ and $\{w'_t\}_{t\in [T]}$ that are respectively generated over $S$ and $S'$ via $A_{\sgdw}$  via sample path $\xi=\{\xi_t\}_{t\in [T]}$. Note by assumption that $w_0=w'_0$. We distinguish the following two complementary cases.

\textbf{Case I: $z_{\xi_t} = z'_{\xi_t}$.} In this case, by invoking Lemma~\ref{lemma:non_expansion_convex} we immediately get
\begin{equation}\label{inequat:distance_expansion_1}
\begin{aligned}
\|w_t - w'_t\|^2 =& \|\Pi_{\mathcal{W}}(w_{t-1} - \eta_t \nabla_w \ell(w_{t-1};z_{\xi_t})) - \Pi_{\mathcal{W}}(w'_{t-1} - \eta_t \nabla_w \ell(w'_{t-1};z'_{\xi_t}))\|^2 \\
\le& \|w_{t-1} - \eta_t \nabla_w \ell(w_{t-1};z_{\xi_t}) - (w'_{t-1} - \eta_t \nabla_w \ell(w'_{t-1};z'_{\xi_t}))\|^2 \\
\le& \|w_{t-1} - w'_{t-1}\|^2.
\end{aligned}
\end{equation}

\textbf{Case II: $z_{\xi_t} \neq z'_{\xi_t}$.} In this case, we have
\begin{equation}\label{inequat:distance_expansion_2}
\begin{aligned}
\|w_t - w'_t\|^2 =& \|\Pi_{\mathcal{W}}(w_{t-1} - \eta_t \nabla f(w)) - \Pi_{\mathcal{W}}(w' - \alpha \nabla f(w'))\|^2 \\
\le& \|w_{t-1} - \eta_t \nabla_w \ell(w_{t-1};z_{\xi_t}) - (w'_{t-1} - \eta_t \nabla_w \ell(w'_{t-1};z'_{\xi_t}))\|^2 \\
\le& \left(\|w_{t-1} - w'_{t-1}\| + \eta_t (\| \nabla_w \ell(w_{t-1};z_{\xi_t})\| + \| \nabla_w \ell(w'_{t-1};z'_{\xi_t})\|)\right)^2\\
\le& \left(\|w_{t-1} - w'_{t-1}\| + 2G\eta_t\right)^2\\
=& \|w_{t-1} - w'_{t-1}\|^2 + 4G\eta_t \|w_{t-1} - w'_{t-1}\| + 4G^2\eta^2_t,
\end{aligned}
\end{equation}
where in the last but inequality we have used $\ell(\cdot;\cdot)$ is $G$-Lipschitz with respect to its first argument.

Let $\beta_t=\beta_t(S,S',\xi) := \mathbf{1}_{\left\{z_{\xi_t}\neq z'_{\xi_t}\right\}}$ be the random indication function associated with event $z_{\xi_t}\neq z'_{\xi_t}$. Based on the recursion forms~\eqref{inequat:distance_expansion_1} and~\eqref{inequat:distance_expansion_2} and the condition $w_0=w'_0$ we can show that for all $t\ge 1$,
\[
\|w_t - w'_t\|^2 \le \sum_{\tau=1}^t 4G\beta_\tau\eta_\tau \|w_{\tau-1} - w'_{\tau-1}\| + \sum_{\tau=1}^t 4G^2\beta_\tau \eta^2_\tau.
\]
Then applying Lemma~\ref{lemma:key_sequence_bound} with simple algebraic manipulation yields
\[
\|w_t - w'_t\|^2 \le 8G^2\left(\sum_{\tau=1}^t \beta_\tau \eta^2_\tau + 4 \left(\sum_{\tau=1}^t \beta_\tau\eta_\tau\right)^2\right).
\]
Since by assumption $S$ and $S'$ differ only in a single element, under the scheme of uniform sampling without replacement, we can see that $\beta_t(S,S',\xi) \sim \texttt{Bernoulli}(1/N)$ and $\{\beta_t(S,S',\xi)\}$ is an i.i.d. sequence of Bernoulli random variables. It follows that
\[
\begin{aligned}
&\mathbb{E}_{\xi_{[t]}}\left[\|w_t - w'_t\|^2 \right] \\
\le& 8G^2\left(\sum_{\tau=1}^t \mathbb{E}_{\xi_{[t]}}[\beta_\tau] \eta^2_\tau + 4 \mathbb{E}_{\xi_{[t]}}\left[\left(\sum_{\tau=1}^t \beta_\tau\eta_\tau\right)^2\right]\right)\\
=& 8G^2\left(\sum_{\tau=1}^t \mathbb{E}_{\xi_{[t]}}[\beta_\tau + 4\beta^2_\tau] \eta^2_\tau + 4 \sum_{\tau\neq \tau'} 1 \mathbb{E}_{\xi_{[t]}}\left[\beta_\tau\beta_{\tau'}\right]\eta_\tau\eta_{\tau'}\right) \\
=& 8G^2\left(\frac{5}{N}\sum_{\tau=1}^t \eta^2_\tau + \frac{4}{N^2} \left(\sum_{\tau=1}^t \eta_\tau\right)^2\right) \le 40G^2\left(\frac{1}{N}\sum_{\tau=1}^T \eta^2_\tau + \frac{1}{N^2} \left(\sum_{\tau=1}^T \eta_\tau\right)^2\right),
\end{aligned}
\]
where we have used $\mathbb{E}_{\xi_t}[\beta_t] = \mathbb{E}_{\xi_t}[\beta^2_t] =\frac{1}{N}$. The convexity of squared Euclidean norm leads to
\[
 \mathbb{E}_{\xi}\left[\|\bar w_T - \bar w'_T\|^2\right] \le \frac{\sum_{t=1}^{T} \mathbb{E}_{\xi_{[t]}}\left[\|w_t - w'_t\|^2\right] }{T}\le 40G^2\left(\frac{1}{N}\sum_{t=1}^T \eta^2_t + \frac{1}{N^2} \left(\sum_{t=1}^T \eta_t\right)^2\right) .
\]
Since the above holds for any $S\doteq S'$, we have that for all $i\in [N]$,
\[
\mathbb{E}_{\xi,S,S^{(i)}} \left[\|\bar w_T - \bar w^{(i)}_T\|^2\right] \le \frac{40G^2}{N}\left(\sum_{t=1}^T \eta^2_t + \frac{1}{N} \left(\sum_{t=1}^T \eta_t\right)^2\right),
\]
where $\{w^{(i)}_t\}_{t\in [T]}$ is generated over $S^{(i)}$ via $A_{\sgdw}$. Finally, since the loss is $G$-Lipschitz, it follows from the above that for all $i\in [N]$,
\[
\left\|\ell(\bar w_T;Z) - \ell(\bar w^{(i)}_T;Z) \right\|_2 \le G \left\|\bar w_T - \bar w^{(i)}_T\right\|_2 \le G^2\sqrt{\frac{40}{N}\left(\sum_{t=1}^T \eta^2_t + \frac{1}{N} \left(\sum_{t=1}^T \eta_t\right)^2\right)}.
\]
This proves the desired $L_2$-stability bound.
\end{proof}

With Lemma~\ref{lemma:uniform_stability_sgd_w_sm_cvx} in place, we are ready to prove Corollary~\ref{corol:stability_sgd_w_convex_smooth}.

\begin{proof}[Proof of Corollary~\ref{corol:stability_sgd_w_convex_smooth}]
From Lemma~\ref{lemma:uniform_stability_sgd_w_sm_cvx} we know that $A_{\sgdw}$ has $L_2$-stability with parameter
\[
\gamma_{\ms, N}=G^2\sqrt{\frac{40}{N}\left(\sum_{t=1}^T \eta^2_t + \frac{1}{N} \left(\sum_{t=1}^T \eta_t\right)^2\right)}.
\]
The desired results then follow immediately via invoking Theorem~\ref{thrm:stability_genalization_error_highprob_bagging} with $\alpha=1/2$.
\end{proof}

\subsection{Proof of Corollary~\ref{corol:stability_sgd_w_nsm_cvx}}
\label{apdsect:proof_lemma:uniform_stability_sgd_w_nsm_cvx}

We first establish the following lemma on the mean(-square)-uniform stability of $A_{\sgdw}$ in the case of non-smooth convex loss.
\begin{lemma}\label{lemma:uniform_stability_sgd_w_nsm_cvx}
Suppose that the loss function is $\ell(\cdot; \cdot)$ is convex and $G$-Lipschitz with respect to its first argument. Then $A_{\sgdw}$ has $L_2$-stability by
\[
\gamma_{\ms,N}= G^2 \sqrt{40\sum_{t=1}^{T}\eta^2_t + \frac{32}{N^2} \left(\sum_{t=1}^T \eta_t\right)^2}.
\]
\end{lemma}
\begin{proof}
Let us define the sequences $\{w_t\}_{t\in[T]}$ and $\{w'_t\}_{t\in [T]}$ that are respectively generated over $S$ and $S'$ via $A_{\sgdw}$  via sample path $\xi=\{\xi_t\}_{t\in [T]}$. Suppose that $S\doteq S'$ and consider a hitting time variable $t_0 = \inf\{ t: z_{\xi_t}\neq z'_{\xi_t}\}$. Let $\beta_t=\beta_t(S,S',\xi) := \mathbf{1}_{\left\{z_{\xi_t}\neq z'_{\xi_t}\right\}}$ be the random indication function associated with event $z_{\xi_t}\neq z'_{\xi_t}$. Then $\{\beta_t\}$ is an i.i.d. sequence of $\texttt{Bernoulli}(1/N)$ random variables. Conditioned on $t_0$, it has been shown by~\citet[Lemma 3.1]{bassily2020stability} that
\begin{equation}\label{inequat:bassily_key}
\|w_t - w'_t\| \le 2G\sqrt{\sum_{\tau=t_0}^{t}\eta^2_\tau} + 4G\sum_{\tau=t_0+1}^t \beta_\tau \eta_\tau \le 2G\sqrt{\sum_{\tau=1}^{t}\eta^2_\tau} + 4G\sum_{\tau=1}^t \beta_\tau \eta_\tau.
\end{equation}
Based on the square of the bound~\eqref{inequat:bassily_key} we can show that
\[
\begin{aligned}
\mathbb{E}_{\xi_{[t]}}\left[\|w_t - w'_t\|^2 \right] \le& \mathbb{E}_{\xi_{[t]}}\left[8G^2\sum_{\tau=1}^{t}\eta^2_\tau + 32G^2\left(\sum_{\tau=1}^t \beta_\tau \eta_\tau\right)^2\right]\\
=& 8G^2\sum_{\tau=1}^{t}\eta^2_\tau + 32G^2 \mathbb{E}_{\xi_{[t]}}\left[\sum_{\tau=1}^t \beta^2_\tau \eta^2_\tau + \sum_{\tau \neq \tau'} \beta_\tau \beta_{\tau'}\eta_\tau\eta_{\tau'}\right]\\
=& 8G^2\sum_{\tau=1}^{t}\eta^2_\tau + 32G^2 \left(\frac{1}{N}\sum_{\tau=1}^t \eta^2_\tau + \frac{1}{N^2}\sum_{\tau \neq \tau'} \eta_\tau\eta_{\tau'}\right)\\
\le& 40G^2\sum_{\tau=1}^{t}\eta^2_\tau + \frac{32G^2}{N^2} \left(\sum_{\tau=1}^t \eta_\tau\right)^2,
\end{aligned}
\]
where we have used $\mathbb{E}_{\xi_t}[\beta_t] = \mathbb{E}_{\xi_t}[\beta^2_t] =\frac{1}{N}$. It follows directly from the convexity of squared loss that
\[
 \mathbb{E}_{\xi_{[T]}}\left[\|\bar w_T - \bar w'_T\|^2\right] \le 40G^2\sum_{t=1}^{T}\eta^2_t + \frac{32G^2}{N^2} \left(\sum_{t=1}^T \eta_t\right)^2.
\]
Since the above holds for any $S\doteq S'$, we have that for all $i\in [N]$,
\[
\mathbb{E}_{\xi,S,S^{(i)}} \left[\|\bar w_T - \bar w^{(i)}_T\|^2\right] \le 40G^2\sum_{t=1}^{T}\eta^2_t + \frac{32G^2}{N^2} \left(\sum_{t=1}^T \eta_t\right)^2,
\]
where $\{w^{(i)}_t\}_{t\in [T]}$ is generated over $S^{(i)}$ via $A_{\sgdw}$. Finally, since the loss is $G$-Lipschitz, it follows from the above that for all $i\in [N]$,
\[
\left\|\ell(\bar w_T;Z) - \ell(\bar w^{(i)}_T;Z) \right\|_2 \le G \left\|\bar w_T - \bar w^{(i)}_T\right\|_2 \le G^2\sqrt{40\sum_{t=1}^{T}\eta^2_t + \frac{32}{N^2} \left(\sum_{t=1}^T \eta_t\right)^2}.
\]
This proves the desired $L_2$-stability bound.
\end{proof}
Equipped with Lemma~\ref{lemma:uniform_stability_sgd_w_nsm_cvx}, we are now in the position to prove Corollary~\ref{corol:stability_sgd_w_nsm_cvx}.
\begin{proof}[Proof of Corollary~\ref{corol:stability_sgd_w_nsm_cvx}]
From Lemma~\ref{lemma:uniform_stability_sgd_w_nsm_cvx} we know that $A_{\sgdw}$ with non-smooth convex loss has $L_2$-stability by
\[
\gamma_{\ms, N} = G^2\sqrt{40\sum_{t=1}^{T}\eta^2_t + \frac{32}{N^2} \left(\sum_{t=1}^T \eta_t\right)^2}.
\]
The desired results then follow immediately via invoking Theorem~\ref{thrm:stability_genalization_error_highprob_bagging} with $\alpha=1/2$.
\end{proof}

\subsection{Proof of Corollary~\ref{corol:stability_sgd_w_sm_ncvx}}
\label{apdsect:proof_lemma:uniform_stability_sgd_w_sm_ncvx}

We first establish the following lemma on the mean(-square)-uniform stability of $A_{\sgdw}$ in the considered non-convex regime.
\begin{lemma}\label{lemma:uniform_stability_sgd_w_sm_ncvx}
Suppose that the loss function is $\ell(\cdot; \cdot)$ is $G$-Lipschitz and $L$-smooth with respect to its first argument. Consider $\eta_t\le 1/L$. For all $t\ge 1$, let
\[
u_t:=\eta^2_t + 2\eta_t\sum_{\tau =1}^{t-1} \exp\left(L\sum_{i=\tau+1}^{t-1} \eta_i \right) \eta_\tau.
\]
Then $A_{\sgdw}$ has $L_2$-stability by
\[
\gamma_{\ms,N} = 2G^2\sqrt{\frac{1}{N} \sum_{t=1}^T \exp\left(3L\sum_{\tau=t+1}^T \eta_\tau  \right) u_t}.
\]
\end{lemma}
\begin{proof}
Let us define the sequences $\{w_t\}_{t\in[T]}$ and $\{w'_t\}_{t\in [T]}$ that are respectively generated over $S$ and $S'$ via $A_{\sgdw}$  via sample path $\xi=\{\xi_t\}_{t\in [T]}$. Suppose that $S\doteq S'$. Let us consider $\Delta_t:=\mathbb{E}_{\xi_{[t]}}\left[\|w_t - w'_t\|\right]$. Then based on the arguments of~\citet[Theorem 3.8]{hardt2016train} we know that with probability $1-\frac{1}{N}$ over $\xi_t$, $\|w_t - w'_t\|\le (1+\eta_t L)\|w_{t-1} - w'_{t-1}\|$, and $\|w_t - w'_t\|\le \|w_{t-1} - w'_{t-1}\| + 2G\eta_t$ with probability $\frac{1}{N}$. Therefore we have
\[
\begin{aligned}
\Delta_t\le& \left(1 - \frac{1}{N}\right)(1+\eta_t L) \Delta_{t-1} + \frac{1}{N} \left(\Delta_{t-1} + 2G\eta_t\right)\\
=& \left(\left(1 - \frac{1}{N}\right)(1+\eta_t L) + \frac{1}{N}\right) \Delta_{t-1 } + \frac{2G\eta_t}{N} \\
=& \left(1 + \left(1-\frac{1}{N} \right)\eta_t L\right) \Delta_{t-1 } + \frac{2G\eta_t}{N} \\
\le& \exp\left(\left( 1-\frac{1}{N} \right)\eta_t L\right)\Delta_{t-1} + \frac{2G\eta_t}{N} \\
\le& \exp\left(\eta_t L\right)\Delta_{t-1} + \frac{2G\eta_t}{N},
\end{aligned}
\]
where we have used $1+x\le \exp(x)$. Then we can unwind the above recursion form to obtain that for all $t\ge 1$,
\begin{equation}\label{inequat:Delta_t_t0}
\Delta_t \le \sum_{\tau =1}^t \left\{\prod_{i=\tau+1}^t \exp\left( \eta_i L\right)\right\} \frac{2G\eta_\tau}{N}
= \frac{2G}{N}\sum_{\tau =1}^t \exp\left(L\sum_{i=\tau+1}^t \eta_i \right) \eta_\tau,
\end{equation}
where we have used $\Delta_{0}=0$. Now we consider $\tilde \Delta_t:=\mathbb{E}_{\xi_{[t]}}\left[\|w_t - w'_t\|^2\right]$. Then we can verify that with probability $1-\frac{1}{N}$ over $\xi_t$, $\|w_t - w'_t\|^2\le (1+\eta_t L)^2\|w_{t-1} - w'_{t-1}\|^2$, and with probability $\frac{1}{N}$,
\[
\|w_t - w'_t\|^2\le (\|w_{t-1} - w'_{t-1}\| + 2G\eta_t)^2 = \|w_{t-1} - w'_{t-1}\|^2 + 4G\eta_t\|w_{t-1} - w'_{t-1}\| + 4G^2\eta_t^2.
\]
Therefore we have
\[
\begin{aligned}
\tilde \Delta_t \le& \left(1 - \frac{1}{N}\right)(1+\eta_t L)^2 \tilde\Delta_{t-1} + \frac{1}{N} \left(\tilde\Delta_{t-1} + 4G\eta_t\Delta_{t-1} + 4G^2\eta_t^2 \right)\\
\le& \left(\left(1 - \frac{1}{N}\right)(1+\eta_t L)^2 + \frac{1}{N}\right) \tilde \Delta_{t-1} + \frac{4G^2}{N}\left(\underbrace{\eta^2_t + 2\eta_t\sum_{\tau =1}^{t-1} \exp\left(L\sum_{i=\tau+1}^{t-1} \eta_i \right) \eta_\tau}_{u_t}\right)\\
=& \left(1 + \left(1-\frac{1}{N} \right)(2\eta_t L + \eta^2_t L^2)\right) \tilde \Delta_{t-1} + \frac{4G^2u_t}{N} \\
\le& \exp\left(\left( 1-\frac{1}{N} \right)(2\eta_t L+\eta^2_tL^2)\right)\tilde \Delta_{t-1} + \frac{4G^2u_t}{N} \\
\le& \exp\left(2\eta_t L+\eta^2_tL^2\right)\tilde \Delta_{t-1} + \frac{4G^2u_t}{N},
\end{aligned}
\]
where in the second inequality we have used the bound~\eqref{inequat:Delta_t_t0} on $\Delta_t$. Recall that $\tilde \Delta_{0}=0$. Then we can unwind the above recursion form to obtain
\[
\tilde\Delta_t \le \frac{4G^2}{N} \sum_{\tau =1}^t \left\{\prod_{i=\tau+1}^t \exp\left( 2\eta_i L + \eta^2_i L^2\right)\right\} u_\tau \le \frac{4G^2}{N}\sum_{\tau =1}^t \exp\left(3L\sum_{i=\tau+1}^t \eta_i \right) u_\tau,
\]
where we have used $\eta_t\le 1/L$.
It follows immediately from the convexity that
\[
 \mathbb{E}_{\xi_{[T]}}\left[\|\bar w_T - \bar w'_T\|^2\right] \le \frac{\sum_{t=1}^{T} \mathbb{E}_{\xi_{[t]}}\left[\|w_t - w'_t\|^2\right] }{T}\le \frac{4G^2}{N} \sum_{t=1}^T \exp\left(3L\sum_{\tau=t+1}^T \eta_\tau  \right) u_t.
\]
Since the above holds for any $S\doteq S'$, we have that for all $i\in [N]$,
\[
\mathbb{E}_{\xi,S,S^{(i)}} \left[\|\bar w_T - \bar w^{(i)}_T\|^2\right] \le \frac{4G^2}{N} \sum_{t=1}^T \exp\left(3L\sum_{\tau=t+1}^T \eta_\tau  \right) u_t,
\]
where $\{w^{(i)}_t\}_{t\in [T]}$ is generated over $S^{(i)}$ via $A_{\sgdw}$. Finally, since the loss is $G$-Lipschitz, it follows from the above that for all $i\in [N]$,
\[
\left\|\ell(\bar w_T;Z) - \ell(\bar w^{(i)}_T;Z) \right\|_2 \le G \left\|\bar w_T - \bar w^{(i)}_T\right\|_2 \le 2G^2\sqrt{\frac{1}{N} \sum_{t=1}^T \exp\left(3L\sum_{\tau=t+1}^T \eta_\tau  \right) u_t}.
\]
This proves the desired $L_2$-stability bound.
\end{proof}
With Lemma~\ref{lemma:uniform_stability_sgd_w_sm_ncvx} in place, we proceed to prove the main result in Corollary~\ref{corol:stability_sgd_w_sm_ncvx}.
\begin{proof}[Proof of Corollary~\ref{corol:stability_sgd_w_sm_ncvx}]
From Lemma~\ref{lemma:uniform_stability_sgd_w_sm_ncvx} we know that $A_{\sgdw}$ with smooth non-convex loss has $L_2$-stability by
\[
 \gamma_{\ms, N}= 2G^2\sqrt{\frac{1}{N} \sum_{t=1}^T \exp\left(3L\sum_{\tau=t+1}^T \eta_\tau  \right) u_t}.
\]
The desired results then follow immediately via invoking Theorem~\ref{thrm:stability_genalization_error_highprob_bagging} with $\alpha=1/2$.
\end{proof}

\begin{algorithm}[b!]
\caption{$A_{\sgdwo}$: SGD under Without-Replacement Sampling}
\label{alg:sgd_wo}
\SetKwInOut{Input}{Input}\SetKwInOut{Output}{Output}\SetKw{Initialization}{Initialization}
\Input{Data set $S=\{Z_i\}_{i\in [N]} \overset{\text{i.i.d.}}{\sim} \mathcal{D}^N$, step-sizes $\{\eta_t\}_{t\ge 1}$, \#iterations $T$, initialization $w_0$.}
\Output{$\bar w_{T}=\frac{1}{T}\sum_{t\in [T]}w_t$.}

\For{$t=1, 2, ...,T$}{

Uniformly randomly sample an index $\xi_t\in [N]$ \emph{with} or \emph{without} replacement;

Compute $w_t = \Pi_{\mathcal{W}}\left( w_{t-1} - \eta_t \nabla_w \ell(w_{t-1}; Z_{\xi_t})\right)$.
}
\end{algorithm}

\section{Augmented Results for SGD under Without-Replacement Sampling}
\label{apdsect:results_for_SGD_wo}

In this section, we further consider applying our main results in Theorem~\ref{thrm:stability_genalization_error_highprob_bagging} to the variant of SGD under without-replacement sampling ($A_{\sgdwo}$), as is outlined in Algorithm~\ref{alg:sgd_wo}. For the sake of simplicity and readability, we only consider single-epoch processing with $T \le N$, and we focus on the case where the loss is convex but non-smooth. The extensions of our analysis to multi-epoch processing, i.e., $T \le rN$ for some integer $r\ge 1$ are more or less straightforward and thus are omit to avoid redundancy.

\subsection{Results for Convex and Smooth Loss}

We start by considering the regime where the loss function is convex and smooth. We need the following lemma on the $L_2$-stability of $A_{\sgdwo}$ which can be easily proved based on the result from~\citet[Lemma 3.1]{bassily2020stability}.

\begin{lemma}\label{lemma:uniform_stability_sgd_wo_sm_cvx}
Suppose that the loss function $\ell(\cdot; \cdot)$ is convex, $G$-Lipschitz and $L$-smooth with respect to its first argument. Assume that $\eta_t\le 2/L$ for all $t\ge 1$. Consider $T\le N$. Then $A_{\sgdwo}$ has $L_2$-stability by
\[
\gamma_{\ms,N} = 2G^2 \sqrt{\frac{1}{N}\sum_{t=1}^T \eta^2_t}.
\]
\end{lemma}
\begin{proof}
Let $\bar w_{T} (S,\xi)$ and $\bar w_{T} (S',\xi)$ respectively be the output generated over $S=\{z_i\}_{i\in [N]}$ and $S'=\{z'_i\}_{i\in [N]}$ by $A_{\sgdwo}$ via sample path $\xi=\{\xi_t\}_{t\in [T]}$. Recall that $T\le N$. Let us define a stopping time variable $t_0$ such that $z_{\xi_{t_0}}\neq z'_{\xi_{t_0}}$. Since $S\doteq S'$, the uniform randomness of $\xi_t$ implies that
\[
\mathbb{P}\left(t_0 = j \right) = \frac{1}{N}, \quad j\in [N].
\]
In the proof of Corollary~\ref{corol:stability_sgd_w_convex_smooth} we have already shown that $\|w_t - w'_t\|^2 \le \|w_{t-1} - w'_{t-1}\|^2$ if $z_{\xi_t} = z'_{\xi_t}$ and $\|w_t - w'_t\|^2\le \|w_{t-1} - w'_{t-1}\|^2 + 4G\eta_t \|w_{t-1} - w'_{t-1}\| + 4G^2\eta^2_t$ otherwise. Therefore, the without-replacement sampling implies that the following bound holds for any given $t_0 \le t \le T$:
\[
\|w_t - w'_t\|^2 \le 4G^2\eta^2_{t_0},
\]
and $\|w_t - w'_t\|^2 = 0$ for $0 \le t<t_0$. Then based on the law of total expectation we can show that
\[
\mathbb{E}_{\xi_{[t]}}\left[\|w_t - w'_t\|^2 \right] \le \frac{4G^2}{N}\sum_{t_0=1}^t \eta^2_{t_0} \le \frac{4G^2}{N}\sum_{t=1}^{T}\eta^2_t.
\]
The convexity of squared Euclidean norm leads to
\[
 \mathbb{E}_{\xi_{[T]}}\left[\|\bar w_T - \bar w'_T\|^2\right] \le \frac{\sum_{t=1}^{T} \eta_t \mathbb{E}_{\xi_{[t]}}\left[\|w_t - w'_t\|^2\right] }{\sum_{t=1}^T\eta_t}\le \frac{4G^2}{N}\sum_{t=1}^T \eta^2_t.
\]
Finally, since the loss is $G$-Lipschitz, it follows from the above that for all $i\in [N]$,
\[
\left\|\ell(\bar w_T;Z) - \ell(\bar w^{(i)}_T;Z) \right\|_2 \le G \left\|\bar w_T - \bar w^{(i)}_T\right\|_2 \le 2G^2\sqrt{\frac{1}{N}\sum_{t=1}^T \eta^2_t}.
\]
This proves the desired $L_2$-stability bound.
\end{proof}

\newpage

The following result is a direct consequence of Theorem~\ref{thrm:stability_genalization_error_highprob_bagging} and Theorem~\ref{thrm:stability_excess_risk_highprob_bagging} when invoking Algorithm~\ref{alg:randomized_model_selection} to $A_{\sgdwo}$ with convex and smooth loss.
\begin{corollary}\label{corol:stability_sgd_wo_convex_smooth}
Suppose that the loss function $\ell(\cdot; \cdot)$ is convex, $G$-Lipschitz and $L$-smooth with respect to its first argument, and is bounded in the range of $[0,M]$. Consider Algorithm~\ref{alg:randomized_model_selection} specified to $A_{\sgdwo}$ with $T= N$ and learning rate $\eta_t\le 2/L$ for all $t\ge 1$. Then for any $\delta \in (0,1)$ and $K \ge 2\log(\frac{4}{\delta}) $, with probability at least $1 - \delta$ over the randomness of $S$ and $\{A_{\sgdwo,k}\}_{k\in [K]}$, the generalization bound of Algorithm~\ref{alg:randomized_model_selection} is upper bounded as
\[
\left|R(A_{\sgdwo,k^*}(S_{k^*})) - R_S(A_{\sgdwo,k^*}(S_{k^*}))\right|  \lesssim  G^2 \sqrt{\frac{1}{N}\sum_{t=1}^N\eta^2_t} + M\sqrt{\frac{\log(K/\delta)}{N}}.
\]
Moreover assume that $\mathcal{W}$ is bounded with diameter $D$. Then with constant learning rate $\eta_t\equiv \min\{\frac{2}{L}, \frac{D}{G\sqrt{N}}\} $, the excess risk of the modified output~\eqref{equat:output_modify_subbagging} of Algorithm~\ref{alg:randomized_model_selection} satisfies
\[
\left|R(A_{\sgdwo,k^*}(S_{k^*})) - R(w^*)\right|  \lesssim \frac{GDK}{\sqrt{N}}  + M\sqrt{\frac{\log(K/\delta)}{N}} + \frac{D^2L}{N}.
\]
\end{corollary}
\begin{proof}
For the considered convex and smooth losses, from Lemma~\ref{lemma:uniform_stability_sgd_wo_sm_cvx} we know that $A_{\sgdwo}$ with $T=N$ iterations has $L_2$-stability by
\[
\gamma_{\ms, N} = 2G^2 \sqrt{\frac{1}{N}\sum_{t=1}^N \eta^2_t}.
\]
The first generalization error bound then follows immediately via invoking Theorem~\ref{thrm:stability_genalization_error_highprob_bagging} with $\alpha=1/2$. The second excess risk bound can be established by invoking Theorem~\ref{thrm:stability_excess_risk_highprob_bagging} with $\alpha=1/2$ and the following in-expectation optimization error bound of $A_{\sgdwo}$ with convex and smooth loss functions under the given constant learning rate~\citep[Theorem 3]{nagaraj19a}:
\[
\mathbb{E} \left[R_S(\bar w_T) - R_S(w^*)\right] \lesssim \frac{D^2 L}{N} + \frac{GD}{\sqrt{N}}.
\]
This complete the proof.
\end{proof}
\begin{remark}
Specially for constant learning rates $\eta_t\equiv \eta \asymp \frac{1}{\sqrt{N}}$ and $K\asymp \log\left(\frac{1}{\delta}\right)$, Corollary~\ref{corol:stability_sgd_wo_convex_smooth} admits a high-probability generalization bound of order $\mathcal{O}\left(\sqrt{\frac{\log(1/\delta)}{N}} + \frac{1}{\sqrt{N}}\right)$ and an excess bounds dominated by $\mathcal{O}\left(\sqrt{\frac{\log(1/\delta)}{N}} + \frac{1}{N}\right)$. For time varying learning rates $\eta_t \propto \frac{1}{\sqrt{t}}$, the generalization bound scales as $\mathcal{O}\left(\sqrt{\frac{\log(N)}{N}} + \sqrt{\frac{\log(1/\delta)}{N}} \right)$.
\end{remark}

\newpage

\subsection{Results for Convex and Non-smooth Loss}

We now turn to study the case of convex and non-smooth losses. The following lemma is about the $L_2$-stability of $A_{\sgdwo}$ in this case.

\begin{lemma}\label{lemma:uniform_stability_sgd_wo_ns_cvx}
Suppose that the loss function $\ell(\cdot; \cdot)$ is convex and $G$-Lipschitz with respect to its first argument. Consider $T\le N$. Then $A_{\sgdwo}$ has $L_2$-stability by
\[
\gamma_{\ms,N} = 2G^2\sqrt{\frac{1}{N}\sum_{t_0=1}^T \sum_{t=t_0}^{T}\eta^2_t}.
\]
\end{lemma}
\begin{proof}
Let $\bar w_{T} (S,\xi)$ and $\bar w_{T} (S',\xi)$ respectively be the output generated over $S=\{z_i\}_{i\in [N]}$ and $S'=\{z'_i\}_{i\in [N]}$ by $A_{\sgdwo}$ via sample path $\xi=\{\xi_t\}_{t\in [T]}$. Recall that $T\le N$. Let us define a stopping time variable $t_0$ such that $z_{\xi_{t_0}}\neq z'_{\xi_{t_0}}$. Since $S\doteq S'$, the uniform randomness of $\xi_t$ and the without-replacement sampling strategy yield
\[
\mathbb{P}\left(t_0 = j \right) = \frac{1}{N}, \quad j\in [N].
\]
For any $t_0 \le t \le T$, it follows from Lemma~\ref{lemma:expansion_convex_nonsm} and without-replacement sampling that
\[
\|w_t - w'_t\|^2 \le 4G^2\sum_{\tau=t_0}^{t}\eta^2_\tau.
\]
We use the convention $\sum_{\tau=t_0}^{t}\eta^2_\tau = 0$ for $0\le t < t_0$. Then according to the law of total expectation we must have
\[
\mathbb{E}_{\xi_{[t]}}\left[\|w_t - w'_t\|^2 \right] \le \frac{4G^2}{N}\sum_{t_0=1}^t \sum_{\tau=t_0}^{t}\eta^2_\tau \le \frac{4G^2}{N}\sum_{t_0=1}^T \sum_{\tau=t_0}^{T}\eta^2_\tau.
\]
The convexity of squared Euclidean norm leads to
\[
 \mathbb{E}_{\xi_{[T]}}\left[\|\bar w_T - \bar w'_T\|^2\right] \le \frac{\sum_{t=1}^{T} \mathbb{E}_{\xi_{[t]}}\left[\|w_t - w'_t\|^2\right] }{T}\le \frac{4G^2}{N}\sum_{t_0=1}^T \sum_{t=t_0}^{T}\eta^2_t.
\]
Since the loss is $G$-Lipschitz, it follows from the above that for all $i\in [N]$,
\[
\left\|\ell(\bar w_T;Z) - \ell(\bar w^{(i)}_T;Z) \right\|_2 \le G \left\|\bar w_T - \bar w^{(i)}_T\right\|_2 \le 2G^2\sqrt{\frac{1}{N}\sum_{t_0=1}^T \sum_{t=t_0}^{T}\eta^2_t}.
\]
This proves the desired $L_2$-stability bound.
\end{proof}

With the above lemma in hand, we can establish the following result as a direct consequence of Theorem~\ref{thrm:stability_genalization_error_highprob_bagging} when invoking Algorithm~\ref{alg:randomized_model_selection} to $A_{\sgdwo}$ with convex and non-smooth losses.
\begin{corollary}\label{corol:stability_sgd_wo_convex_nonsmooth}
Suppose that the loss function $\ell(\cdot; \cdot)$ is convex and $G$-Lipschitz with respect to its first argument, and is bounded in the range of $[0,M]$. Consider Algorithm~\ref{alg:randomized_model_selection} specified to $A_{\sgdwo}$ with $T = N$. Then for any $\delta \in (0,1)$ and $K \ge 2\log(\frac{4}{\delta}) $,  with probability at least $1 - \delta$ over the randomness of $S$ and $\{A_{\sgdwo,k}\}_{k\in [K]}$, the output of Algorithm~\ref{alg:randomized_model_selection} satisfies the generalization bound
\[
\left|R(A_{\sgdwo,k^*}(S_{k^*})) - R_S(A_{\sgdwo,k^*}(S_{k^*}))\right| \lesssim G^2 \sqrt{\frac{1}{N}\sum_{t_0=1}^N \sum_{t=t_0}^{N}\eta^2_t} + M\sqrt{\frac{\log(K/\delta)}{N}}.
\]
\end{corollary}
\begin{proof}
For the considered convex and Lipschitz loss functions, from Lemma~\ref{lemma:uniform_stability_sgd_wo_ns_cvx} we know that $A_{\sgdwo}$ with $T=N$ iterations has $L_2$-stability by
\[
\gamma_{\ms, N} = 2G^2\sqrt{\frac{1}{N}\sum_{t_0=1}^N \sum_{t=t_0}^{N}\eta^2_t}.
\]
The results then follow immediately via invoking Theorem~\ref{thrm:stability_genalization_error_highprob_bagging} with $\alpha=1/2$.
\end{proof}
\begin{remark}
Specially for constant learning rates $\eta_t\equiv \eta$ and setting $K\asymp \log\left(\frac{1}{\delta}\right)$, Corollary~\ref{corol:stability_sgd_wo_convex_nonsmooth} admits a high-probability generalization bound of scale
\[
\mathcal{O}\left(\eta \sqrt{N} + \sqrt{\frac{\log(1/\delta)}{N}}\right).
\]
For time decaying learning rates $\eta_t \propto \frac{1}{t}$, the generalization bound scales as
\[
\mathcal{O}\left(\sqrt{\frac{\log(N)}{N}} + \sqrt{\frac{\log(1/\delta)}{N}} \right)
\]
with detail bound $\delta$.
\end{remark}
\begin{remark}
Regarding the excess risk bound, under the conditions of Corollary~\ref{corol:stability_sgd_wo_convex_nonsmooth} and $K\asymp \log\left(\frac{1}{\delta}\right)$, Theorem~\ref{thrm:stability_excess_risk_highprob_bagging} combined with Lemma~\ref{lemma:uniform_stability_sgd_wo_ns_cvx} yields that the modified output of Algorithm~\ref{alg:randomized_model_selection} as defined by~\eqref{equat:output_modify_subbagging} satisfies the following bound with probability at least $1-\delta$:
\[
\left|R(A_{\sgdwo,k^*}(S_{k^*})) - R(w^*)\right|  \lesssim G^2\sqrt{\frac{\log(1/\delta)}{N}\sum_{t_0=1}^N \sum_{t=t_0}^N \eta^2_t} + \Delta_{\opt}   + M\sqrt{\frac{\log(1/\delta)}{N}}.
\]
In the special case of bounded-norm generalized linear models, \citet{shamir2016without} established an in-expectation empirical risk sub-optimality bound $\Delta_{\opt}  \lesssim \frac{1}{\sqrt{N}}$ under suitable learning rates. For generic convex and non-smooth losses, however, it still remains unclear to us if similar sub-optimality bounds are available for SGD under without-replacement sampling.
\end{remark}
\subsection{Results for Non-convex and Smooth Loss}

Finally, we study the performance of Algorithm~\ref{alg:randomized_model_selection} for $A_{\sgdwo}$ on smooth but not necessarily convex loss functions. We first establish the following lemma on the mean(-square)-uniform stability of $A_{\sgdwo}$ in the considered non-convex regime.
\begin{lemma}\label{lemma:uniform_stability_sgd_wo_sm_ncvx}
Suppose that the loss function $\ell(\cdot; \cdot)$ is $G$-Lipschitz and $L$-smooth with respect to its first argument. Consider $T\le N$. Then $A_{\sgdwo}$ has $L_2$-stability by
\[
\gamma_{\ms, N} = 2G^2 \sqrt{\frac{1}{N} \sum_{t=1}^T \exp\left(2L\sum_{\tau=t+1}^T \eta_\tau  \right) \eta^2_t}.
\]
\end{lemma}
\begin{proof}
Let us define the sequences $\{w_t\}_{t\in[T]}$ and $\{w'_t\}_{t\in [T]}$ that are respectively generated over $S$ and $S'$ via $A_{\sgdwo}$  via sample path $\xi=\{\xi_t\}_{t\in [T]}$. Suppose that $S\doteq S'$. Recall that $T\le N$. Let us define a stopping time variable $t_0$ such that $z_{\xi_{t_0}}\neq z'_{\xi_{t_0}}$. Since $S\doteq S'$, the without-replacement sampling implies that
\[
\mathbb{P}\left(t_0 = j \right) = \frac{1}{N}, \quad j\in [N].
\]
In view of Lemma~\ref{lemma:non_expansion_convex} we know that $\|w_t - w'_t\|\le (1+\eta_t L)\|w_{t-1} - w'_{t-1}\|$ if $Z_{\xi_t} = Z'_{\xi_t}$, and $\|w_{t_0} - w'_{t_0}\|\le  2G\eta_{t_0}$ otherwise due to the assumption that the loss is $G$-Lipschitz. Therefore, it can be directly verified that the following holds for any $t_0 \le t \le T$:
\[
\left\|w_t - w'_t\right\|^2 \le 4G^2\prod_{\tau=t_0+1}^{t}(1+\eta_{\tau} L )^2\eta^2_{t_0},
\]
where we have used $\prod_{\tau=t_0+1}^{t}(1+\eta_{\tau}L)^2 = 1$ for $t=t_0$. For $t<t_0$, it is trivial that $\|w_t - w'_t\|=0$. Therefore the law of total expectation yields
\[
\begin{aligned}
\mathbb{E}_{\xi_{[t]}}\left[\|w_t - w'_t\|^2 \right] \le& \frac{4G^2}{N}\sum_{t_0=1}^t \prod_{\tau=t_0+1}^{t}(1+\eta_{\tau} L )^2\eta^2_{t_0} \\
\le& \frac{4G^2}{N}\sum_{t_0=1}^t \prod_{\tau=t_0+1}^{t}\exp\left(2\eta_{\tau} L \right)\eta^2_{t_0} \\
\le& \frac{4G^2}{N}\sum_{t_0=1}^T \exp\left(2L\sum_{\tau=t_0+1}^{T}\eta_{\tau} \right)\eta^2_{t_0} ,
\end{aligned}
\]
where we have used $1+x\le \exp(x)$. The convexity of Euclidean norm leads to
\[
 \mathbb{E}_{\xi_{[T]}}\left[\|\bar w_T - \bar w'_T\|\right] \le \frac{\sum_{t=1}^{T} \eta_t \mathbb{E}_{\xi_{[t]}}\left[\|w_t - w'_t\|\right] }{\sum_{t_0=1}^T \eta_t} \le \frac{4G^2}{N} \sum_{t_0=1}^T \exp\left(2L\sum_{\tau=t_0+1}^T \eta_\tau \right) \eta^2_{t_0}.
\]
Since the loss is $G$-Lipschitz, it follows from the above that for all $i\in [N]$,
\[
\left\|\ell(\bar w_T;Z) - \ell(\bar w^{(i)}_T;Z) \right\|_2 \le G \left\|\bar w_T - \bar w^{(i)}_T\right\|_2 \le 2G^2\sqrt{\frac{1}{N}\sum_{t_0=1}^T \exp\left(2L\sum_{\tau=t_0+1}^T \eta_\tau \right) \eta^2_{t_0}}.
\]
This proves the desired $L_2$-stability bound.
\end{proof}

With Lemma~\ref{lemma:uniform_stability_sgd_wo_sm_ncvx} in place, we can derive the following result as a direct application of Theorem~\ref{thrm:stability_genalization_error_highprob_bagging} to $A_{\sgdwo}$ with Lipschitz and smooth losses. 
\begin{corollary}\label{corol:stability_sgd_wo_sm_ncvx}
Suppose that the loss function $\ell(\cdot; \cdot)$ is $G$-Lipschitz and $L$-smooth with respect to its first argument, and is bounded in the range of $[0,M]$. Consider Algorithm~\ref{alg:randomized_model_selection} specified to $A_{\sgdwo}$ with $T=N$. Then for any $\delta \in (0,1)$ and $K \ge 2\log(\frac{4}{\delta}) $,  with probability at least $1 - \delta$ over the randomness of $S$ and $\{A_{\sgdwo,k}\}_{k\in [K]}$, the output satisfies
\[
\left|R(A_{\sgdwo,k^*}(S_{k^*})) - R_S(A_{\sgdwo,k^*}(S_{k^*}))\right| \lesssim G^2\sqrt{\frac{1}{N} \sum_{t=1}^N \exp\left(L\sum_{\tau=t+1}^N \eta_\tau \right) \eta^2_t} + M\sqrt{\frac{\log(K/\delta)}{N}}.
\]
\end{corollary}
\begin{proof}
For the considered smooth and Lipschitz loss functions, from Lemma~\ref{lemma:uniform_stability_sgd_wo_sm_ncvx} we know that $A_{\sgdwo}$ with $T=N$ rounds of iteration has $L_2$-stability by
\[
\gamma_{\ms, N} = 2G^2 \sqrt{\frac{1}{N} \sum_{t=1}^T \exp\left(2L\sum_{\tau=t+1}^T \eta_\tau  \right) \eta^2_t}.
\]
The desired results then follow immediately via invoking Theorem~\ref{thrm:stability_genalization_error_highprob_bagging} with $\alpha=1/2$.
\end{proof}
\begin{remark}
For $K\asymp\log\left(\frac{1}{\delta}\right)$ and the choice of constant learning rates $\eta_t\equiv \frac{1}{LN}$, Corollary~\ref{corol:stability_sgd_wo_sm_ncvx} admits high-probability generalization bounds of scale $\mathcal{O}\left(\frac{1}{N} + \sqrt{\frac{\log(1/\delta)}{N}}\right)$. For the choice of time decaying learning rates $\eta_t= \frac{1}{L\nu t}$ with arbitrary $\nu>2$, it can be verified that the corresponding generalization bound is of scale
\[
\mathcal{O}\left(\frac{1}{\nu N^{1/2-1/\nu}} + \sqrt{\frac{\log(1/\delta)}{N}}\right).
\]
\end{remark}

\end{document}